\documentclass{article}

\usepackage[preprint]{neurips_2026}

\usepackage{color-edits}
\addauthor{hl}{blue}
\usepackage{soul}

\usepackage[utf8]{inputenc} 
\usepackage[T1]{fontenc}    
\usepackage{hyperref}       
\usepackage{url}            
\usepackage{booktabs}       
\usepackage{amsfonts}       
\usepackage{nicefrac}       
\usepackage{microtype}      
\usepackage{xcolor}         
\usepackage{color-edits}
\addauthor{hl}{blue}
\usepackage{soul}
\usepackage{multirow}
\usepackage[utf8]{inputenc} 
\usepackage[T1]{fontenc}    
\usepackage{hyperref}       
\usepackage{url}            
\usepackage{booktabs}       
\usepackage{amsfonts}       
\usepackage{nicefrac}       
\usepackage{microtype}      
\usepackage{xcolor}         
\usepackage{graphicx} 
\usepackage{amsmath}
\usepackage{mathtools}
\usepackage{caption}   
\usepackage{xcolor}
\usepackage{graphicx}
\usepackage{subcaption}
\usepackage{algorithm}
\usepackage{algpseudocode}
\usepackage{amsthm}
\newtheorem{theorem}{Theorem}
\newtheorem{assumption}{Assumption}
\usepackage{enumitem}
\usepackage{dsfont}
\usepackage{amsthm}
\newtheorem{definition}{Definition}
\newtheorem{proposition}{Proposition}
\newtheorem{remark}{Remark}

\newtheorem{lemma}{Lemma}
\title{Safe and Generalizable Hierarchical Multi-Agent RL via  Constraint Manifold Control}

\author{
  Zihao Guo\textsuperscript{1}\thanks{Corresponding to Zihao Guo $\langle$\texttt{zihao.1.guo@kcl.ac.uk}$\rangle$, Yali Du $\langle$\texttt{yali.du@kcl.ac.uk}$\rangle$.}\quad
  Jianing Zhao\textsuperscript{1}\quad
  Ling Li\textsuperscript{1}\quad
  Hao Liang\textsuperscript{1}\quad
  Giuseppe Loianno\textsuperscript{2}\quad
  Yali Du\textsuperscript{1,3}\footnotemark[1]\\
  \textsuperscript{1}King's College London\qquad
  \textsuperscript{2}University of California, Berkeley\qquad
  \textsuperscript{3}The Alan Turing Institute
}

\begin{document}

\maketitle

\begin{abstract}

Multi-agent systems are widely used in safety-critical applications that require coordinated behavior under strict safety constraints. 
Existing approaches face a fundamental trade-off: learning-based methods achieve strong empirical performance but lack theoretical safety guarantees, while control-theoretic methods enforce safety but often lead to overly conservative and inefficient behaviors.
We propose a hierarchical multi-agent reinforcement learning framework that enforces hard safety constraints under mild assumptions at low level via a constraint manifold, while enabling effective coordination through high-level policy learning. Our approach provides theoretical safety guarantees in the multi-agent setting and yields stationary learning dynamics, thereby enabling stable and efficient training. Empirically, our method achieves competitive performance while maintaining nearly perfect safety rates, and generalizes effectively to varying numbers of agents and obstacles.

\end{abstract}

\section{Introduction}

Multi-agent systems (MAS) have seen significant development in real-world applications such as warehouse robotics \citep{kattepur2018warehouses}, autonomous vehicles \citep{zhang2024autonomous_driving}, traffic routing \citep{wu2020multi_routing}, and drone swarm coordination \citep{batra2022drone}. A common characteristic of these scenarios is that each agent must not only accomplish its own task but also collaborate efficiently with other agents while maintaining safety such as collision avoidance. To address these challenges, various methods have been explored, ranging from Lagrangian-based Constrained Markov Decision Processes (CMDPs) \citep{gu2023safe, liu2021cmix, ding2023provably, lu2021decentralized, geng2023reinforcement, zhao2023multi} that penalize unsafe behaviors, to safety-filter methods such as Control Barrier Functions (CBFs) \citep{zhang2025gcbf+, zhang2025discrete}, traditional control-theoretic approaches like Model Predictive Control (MPC) \citep{NMPC}. 


Recent learning-based CBF approaches for multi-agent systems \citep{zhang2025gcbf+, zhang2025discrete} have demonstrated strong empirical performance by integrating learning with implicit planning capabilities. This enables more effective and scalable cooperative policies in safety-critical tasks such as multi-agent navigation and collision avoidance. Compared to traditional model-based CBF methods with fixed analytical controllers, these approaches further improve adaptability and performance. However, similar to CMDP-based approaches, these methods rely on learned policies and therefore generally lack formal safety guarantees.

In addition to learning-based methods, control-theoretic approaches \citep{NMPC} are proposed to guarantee safety in MAS under hard constraints. However, in complex MAS with multiple obstacles, they often lead to poor coordination among agents and suboptimal path selection, which can result in deadlocks.
To overcome this challenge, safe hierarchical MARL approaches \citep{ahmad2025hmarl} decompose control into a high-level policy that learns to select or switch between skills, and a low-level controller that executes each skill via a parameterized CBF-quadratic programming (QP) controller. In this framework, CBF parameters are learned, but the safety constraints retain a model-based analytical form. 

The need to solve a QP problem at every timestep introduces significant computational overhead, leading to reduced training efficiency and suboptimal performance.
In the single-agent setting, constraint manifold approach ~\citep{manifold} ensures safety via a safe action space, while avoiding QP solves and thus enabling efficient per-timestep computation. 
However, it is unclear whether this approach can be extended to multi-agent settings. 

Existing methods fail to simultaneously strong task performance, 
theoretical safety guarantees, and training efficiency. To bridge these gaps, we propose a novel framework \textbf{H}ierarchical \textbf{M}anifold \textbf{M}ulti-Agent PPO (\textbf{HMM}), which decouples cooperation and safety across two levels: a learnable policy for coordination at high-level and a model-based controller for enforcing hard constraints at low-level. The high-level policy operates under the Centralized Training with Decentralized Execution (CTDE) fashion to sequentially generate subgoals for each agent, handling multi-agent coordination and path planning.  At the low level, the controller guarantees safety by embedding hard constraints into a differentiable constraint manifold and restricting actions to its tangent space. Compared to CBF-based methods that require solving a QP at every timestep, our \textbf{HMM} only involves efficient tangent space projections, while still providing structural guarantees of constraint satisfaction at each environmental timestep under mild assumptions. Our main contributions are summarized as follows:

\begin{itemize}[leftmargin=*]

\item We propose \textbf{HMM}, a hierarchical MARL framework based on a constraint manifold formulation, enabling coordinated multi-agent control while providing formal safety guarantees at every timestep during both training and execution.

\item Our approach provides theoretical safety guarantees in the multi-agent setting, and induces a stationary learning process with stable convergence behavior.

\item  Empirically, \textbf{HMM} outperforms other safety MARL baselines on Lidar benchmarks, achieving state-of-the-art performance. It also demonstrates strong generalization across task scales: policies trained with only 3 agents and 3 obstacles generalize effectively to scenarios involving up to 21 agents or obstacles, while maintaining almost perfect safety rates and high task success rates.

\end{itemize}

\section{Related Work}

Our method lies at the intersection of safe RL and MARL, hierarchical MARL, and constraint manifold based safe control. We review these areas below and highlight that no existing approach provides hard safety guarantees, avoids per step QP solves, and scales to multi agent hierarchical settings. This gap motivates our work.

\textbf{Safe RL and MARL.}
In safe RL, CMDP serves as a classic framework \citep{zhao2023multi, altman2021constrainedbook, gu2023safe} for maximizing cumulative reward while satisfying safety constraints. 
Existing methods can be broadly categorized into primal approaches \citep{xu2021crpo, chow2018lyapunov, chow2019lyapunov, liu2020ipo}, which directly enforce constraints; primal-dual approaches \citep{he2023autocost_dual, borkar2005actor_dual, ding2020natural_dual, huang2023safedreamer_dual, tessler2018reward_dual} that introduce Lagrange multipliers; and trust-region-based approaches \citep{achiam2017constrained_trust_region, he2023autocost_dual}. Among these, Lagrange-multiplier-based approaches \citep{gu2023safe, liu2021cmix, ding2023provably, lu2021decentralized, geng2023reinforcement, zhao2023multi} have been particularly popular and have been extended to the multi-agent setting. However, these Lagrangian-based approaches, whether in single-agent or multi-agent settings, suffer from problems such as training instability, slow convergence, and sensitivity to hyperparameters \citep{zanon2020safe, he2023autocost_dual, so2023solving, ganai2023iterative}. Recently, CBF have been integrated into the RL training process as a safety filter \citep{tearle2021predictive, hsu2023safety, hailemichael2023optimal}. This approach has been explored in both single-agent \citep{cheng2019end, emam2022safe, hailemichael2023optimal} and multi-agent settings \citep{pereira2021safe, pereira2022decentralized, zhang2025discrete}. However, existing CBF-based methods either require solving a QP at each timestep, leading to high computational cost, or rely on learning a parameterized CBF, which fail to provide theoretical safety guarantees.

\textbf{Hierarchical MARL.}
Hierarchical reinforcement learning (HRL) \citep{pertsch2021accelerating} addresses complex, long-horizon tasks by decomposing decision-making into multiple levels of temporal abstraction. HRL has been studied extensively in both single-agent \citep{parr1997reinforcement, dietterich2000hierarchical, sutton1999between} and multi-agent settings \citep{xu2023haven, ghavamzadeh2006hierarchical, son2019qtran, dietterich2000hierarchical, ahilan2019feudal, tessler2017deep}. However, these hierarchical frameworks generally lack explicit safety guarantees. To address this, HMARL-CBF~\citep{ahmad2025hmarl} uses CBF-based low-level controllers for safety. However, it requires solving a QP at each timestep, leading to inefficiency and suboptimal performance compared to DGPPO \citep{zhang2025discrete} on the same benchmarks, as reported in \citet{ahmad2025hmarl}.

\textbf{Constraint Manifold-Based Safe Control.}
Constraint manifold-based methods \citep{liu2022robot, manifold} offer an alternative approach to safety enforcement without requiring QP solvers in RL. This approach has been extended to dynamic high-dimensional robotic tasks \citep{liu2023safe}, stochastic constraint manifolds \citep{gu2024roscom}, and long-term safety under unknown constraints \citep{gunster2024handling}. However, existing works are primarily focused on single-agent settings. In this work, we extend constraint manifolds to multi-agent settings within a hierarchical framework. This enables efficient and scalable safety enforcement while preserving formal guarantees in complex multi-agent environments.

\section{Preliminaries and Problem Formulation}

\subsection{Constrained Semi-Markov Decision Process}

We formulate the multi-agent safe control problem as a constrained semi-Markov decision process (CSMDP)~\citep{makar2001hierarchical}, which allows temporally extended actions. In our setting, we consider the special case where each action is executed for a fixed duration $\tau$.
The CSMDP is defined by the tuple $\langle \mathcal{N}, \mathcal{S}, \mathcal{Z}, \mathcal{O}, \mathcal{P}, r, \mathcal{H}, \gamma \rangle$, where $\mathcal{N} = \{1, \ldots, N\}$ denotes the set of agents.
The global state space is denoted by $\mathcal{S}$. The global state at timestep $t$ is $\mathbf{s}^t = \{\mathbf{s}_i^t\}_{i \in \mathcal{N}}$, where $\mathbf{s}_i^t \in \mathbb{R}^S$ denotes the state of agent $i$. The global state is not directly observable by individual agents; instead, each agent $i$ receives a local observation $o_i^t \in \mathcal{O}_i$.
The joint action and observation spaces are 
$\mathcal{Z} = \mathcal{Z}_1 \times \cdots \times \mathcal{Z}_N$ and 
$\mathcal{O} = \mathcal{O}_1 \times \cdots \times \mathcal{O}_N$, respectively.
Each agent $i \in \mathcal{N}$ has $L$ associated constraint functions defined on local observations, with the joint set defined as
$\mathcal{H} := \{h_i^{(\ell)} : \mathcal{O}_i \rightarrow \mathbb{R} \mid i \in \mathcal{N},\, \ell = 1, \ldots, L\}$.
We require hard constraint satisfaction at every timestep $t$:
\begin{equation}
    h_i^{(\ell)}(o_i^t) \leq 0, \quad \forall i \in \mathcal{N},\ 
    \ell = 1, \ldots, L,\ \forall t \geq 0.
\end{equation}

The multistep transition kernel $\mathcal{P}: \mathcal{S} \times \mathcal{Z} \rightarrow \Delta(\mathcal{S})$, where $\mathcal{P}(\mathbf{s}' \mid \mathbf{s}, \mathbf{z})$ denotes the probability of transitioning to $\mathbf{s}'$ after $\tau$ timesteps under joint action $\mathbf{z}$.
The reward function $r: \mathcal{S} \times \mathcal{Z} \rightarrow \mathbb{R}$ denotes the cumulative reward collected over the $\tau$ timesteps induced by the joint action.

We distinguish between timesteps and decision epochs: the system evolves at the finer timestep index $t$, while decisions are made at discrete decision epochs indexed by $k$. Each decision epoch $k$ corresponds to the execution of a joint action $\mathbf{z}^k$, which is held fixed for $\tau$ consecutive timesteps, i.e., from $t = k\tau$ to $t = (k+1)\tau - 1$. We use 
$\mathbf{s}^k := \mathbf{s}^{k\tau}$ to denote the state at the start of epoch $k$.
The objective is to find decentralized policies $\pi_i: \mathcal{O}_i \rightarrow \mathcal{Z}_i$ defined at decision epochs that maximize the expected discounted return while satisfying the safety constraints:
\begin{equation}
    \max_{\pi_1, \ldots, \pi_N} \mathbb{E}_{\boldsymbol{\pi}}\!\left[
        \sum_{k=0}^{\infty} \gamma^{k\tau}\, r(\mathbf{s}^{k}, \mathbf{z}^{k})
    \right]
    \quad \text{s.t.} \quad
    h_i^{(\ell)}(o_i^t) \leq 0,\ \forall i \in \mathcal{N},\ \ell = 1,\ldots,L,\ \forall t \geq 0.
\end{equation}

The value function under joint policy $\boldsymbol{\pi}$ in the CSMDP is defined as
\begin{equation}
    V^{\boldsymbol{\pi}}(\mathbf{s}) = \mathbb{E}_{\boldsymbol{\pi}}\!\left[
        \sum_{k=0}^{\infty} \gamma^{k\tau}\, r(\mathbf{s}^{k}, \mathbf{z}^{k}) 
        \,\Big|\, \mathbf{s}^0 = \mathbf{s}
    \right].
\end{equation}
This yields an effective discount factor of $\gamma^\tau$ at the decision level.


\subsection{Single Agent Safe Control on the Constraint Manifold}
\label{sec:single_agent_manifold}

In this work, we consider a control-affine system that is locally Lipschitz continuous and defined as follow:

\begin{equation}
    \dot{\mathbf{s}} = f(\mathbf{s}) + G(\mathbf{s})\mathbf{u}_s
    \label{eq:ori_sys_dynamics},
\end{equation}

where $\mathbf{s} \in \mathcal{S} \subset \mathbb{R}^S$ denotes the system state, $\mathbf{u} \in \mathcal{U} \subset \mathbb{R}^U$ denotes the control input, and $f: \mathcal{S} \to \mathbb{R}^S$ and $G: \mathcal{S} \to \mathbb{R}^{S \times U}$ are Lipschitz continuous mappings. We review the key concepts of manifold constraint method~\citep{liu2022robot, manifold} for single agent, which forms the basis of our low-level safe controller. Additional details on manifolds are provided in Appendix~\ref{app:constraint_manifold}.



\textbf{Constraint manifold.} We consider that the safety of the single agent setting requirements is specified by a set of $L$ inequality constraints $h(\mathbf{s}) \leq \mathbf{0}$ with $h: \mathcal{S} \rightarrow \mathbb{R}^L$.
Given the inequality constraint definition, we define the safe set as the 
sublevel set 
$\mathcal{C} := \{\mathbf{s} \in \mathcal{S} \subset \mathbb{R}^S : h(\mathbf{s}) \leq \mathbf{0}\}$. Slack variables $\boldsymbol{\mu} \in [0, +\infty)^L$ are introduced to rewrite the inequality constraints as equality constraints:
\begin{equation}
    c(\mathbf{s}, \boldsymbol{\mu}) := h(\mathbf{s}) + \boldsymbol{\mu} = \mathbf{0}.
\end{equation}

The Jacobian of $c(\mathbf{s}, \boldsymbol{\mu})$ is given by 
$J_c(\mathbf{s}, \boldsymbol{\mu}) = [J_h(\mathbf{s}) \ \ \mathbb{I}_L]$, 
where $\mathbb{I}_L$ denotes the $L$-dimensional identity matrix. The equality constraint formulation defines the constraint manifold $\mathcal{M} := \{(\mathbf{s}, \boldsymbol{\mu}) \in \mathcal{D} : c(\mathbf{s}, \boldsymbol{\mu}) = \mathbf{0}\}$, in the augmented state space $\mathcal{D} := \mathcal{S} \times [0, +\infty)^L \subset \mathbb{R}^{S+L}$. 
The safe set $\mathcal{C}$ is the projection of $\mathcal{M}$ onto the original state space: for any $(\mathbf{s}, \boldsymbol{\mu}) \in \mathcal{M}$, $h(\mathbf{s}) = -\boldsymbol{\mu} \leq \mathbf{0}$, so 
$\mathbf{s} \in \mathcal{C}$. More preliminaries of manifold are provided in Appendix~\ref{app:preliminaries}


\textbf{Augmented dynamics.} To make the slack variables controllable, a controlled system is introduced for each slack variable $\mu_\ell$:
\begin{equation}
    \dot{\mu}_\ell = \alpha_\ell(\mu_\ell) \, u_{\mu, \ell}, \quad \ell = 1, \ldots, L,
    \label{eq:slack_variables}
\end{equation}
where $\alpha_l$ is a class-$\mathcal{K}$\footnote{A continuous function $\alpha: [0, \infty) \rightarrow [0, \infty)$ is of class $\mathcal{K}$ if it is strictly increasing and $\alpha(0) = 0$.} function and $u_{\mu, \ell}$ is the virtual control input for the $\ell$-th slack variable. Combining the original system dynamics~\eqref{eq:ori_sys_dynamics} with the slack variable dynamics~\eqref{eq:slack_variables}, we obtain the augmented system:
\begin{equation}
    \begin{bmatrix} \dot{\mathbf{s}} \\ \dot{\boldsymbol{\mu}} \end{bmatrix} = \begin{bmatrix} f(\mathbf{s}) \\ \mathbf{0} \end{bmatrix} + \begin{bmatrix} G(\mathbf{s}) & \mathbf{0} \\ \mathbf{0} & A(\boldsymbol{\mu}) \end{bmatrix} \begin{bmatrix} \mathbf{u}_s \\ \mathbf{u}_\mu \end{bmatrix},
    \label{eq:augmented_dynamics}
\end{equation}
where $A(\boldsymbol{\mu}) \in \mathbb{R}^{L \times L}$ is a diagonal matrix with entries $A_{\ell \ell} = \alpha_\ell(\mu_\ell)$, and $\mathbf{u}_\mu = [u_{\mu,1}, \ldots, u_{\mu,L}]^\top$ collects all virtual control inputs.


\textbf{Safe controller.} Then we can have a controller that keeps the augmented state $(\mathbf{s}, \boldsymbol{\mu})$ on the constraint manifold $\mathcal{M}$. Since any point on $\mathcal{M}$ projects to a safe state in the original state space, staying on $\mathcal{M}$ guarantees that the system remains within the safe set $\mathcal{C}$.

To keep the augmented state on $\mathcal{M}$, its velocity must lie in the tangent space, which is equivalent to enforcing
\begin{equation}
    \dot{c}(\mathbf{s}, \boldsymbol{\mu}) = J_c(\mathbf{s}, \boldsymbol{\mu}) \begin{bmatrix} \dot{\mathbf{s}} \\ \dot{\boldsymbol{\mu}} \end{bmatrix} = \mathbf{0},
    \label{eq:linear_equation}
\end{equation}
where $J_c(\mathbf{s}, \boldsymbol{\mu}) = [J_h(\mathbf{s})\ \ \mathbb{I}_L]$ is the Jacobian of $c$. Substituting the augmented dynamics~\eqref{eq:augmented_dynamics} into~\eqref{eq:linear_equation} yields the linear equation
\begin{equation}
    \boldsymbol{\psi}(\mathbf{s}) + J_u(\mathbf{s}, \boldsymbol{\mu}) \begin{bmatrix} \mathbf{u}_s \\ \mathbf{u}_\mu \end{bmatrix} = \mathbf{0},
    \label{eq:linear_system}
\end{equation}
where $J_u(\mathbf{s}, \boldsymbol{\mu}) = [J_h(\mathbf{s})G(\mathbf{s})\ \ A(\boldsymbol{\mu})]$ is the control Jacobian and $\boldsymbol{\psi}(\mathbf{s}) = J_h(\mathbf{s})f(\mathbf{s})$ is the \emph{constraint drift} induced by the system drift $f(\mathbf{s})$.

Solving~\eqref{eq:linear_system} gives the general form of the constraint manifold safe controller:
\begin{equation}
    \begin{bmatrix} \mathbf{u}_s \\ \mathbf{u}_\mu \end{bmatrix} = \underbrace{-J_u^\dagger \boldsymbol{\psi}}_{\substack{\text{drift} \\ \text{compensation}}} \underbrace{- \lambda J_u^\dagger c}_{\substack{\text{manifold} \\ \text{contraction}}} + \underbrace{B_u \mathbf{u}}_{\substack{\text{tangential} \\ \text{action}}}
    \label{eq:atacom},
\end{equation}
where $J_u^\dagger$ is the pseudoinverse of $J_u$, $B_u$ is the tangent space basis satisfying $J_u B_u = \mathbf{0}$, $\lambda > 0$ is a contraction gain, and $\mathbf{u}$ is an arbitrary task-level control signal. The three terms play distinct roles: the \emph{drift compensation} term cancels the natural drift of the system, the \emph{manifold contraction} term retracts the augmented state back to $\mathcal{M}$ whenever numerical errors cause deviation, and the \emph{tangential action} term allows free exploration along the tangent space without violating constraints. Any task-level control signal $\mathbf{u}$ mapped through~\eqref{eq:atacom} produces a safe control input $\mathbf{u}_s$, ensuring that the system state remains within the safe set at every timestep.

\section{Methodology}

\subsection{Overview}

We propose a hierarchical framework for safe MARL, illustrated in Figure~\ref{fig:subgoal_idea}. The framework consists of a high-level cooperative planner with a low-level safe controller built on the constraint manifold. The high-level policy generates subgoals for each agent every $\tau$ timesteps to enable coordination. At the low level, each agent independently applies a constraint manifold safe controller that projects low level actions onto the tangent space of its constraint manifold, structurally guaranteeing safety at every timestep. As a result, our framework enables efficient multi-agent coordination while guaranteeing safety by construction at every timestep.

\begin{figure}[htbp]  
    \centering
    \vspace{-1.0em}
    \includegraphics[width=1.0\textwidth]{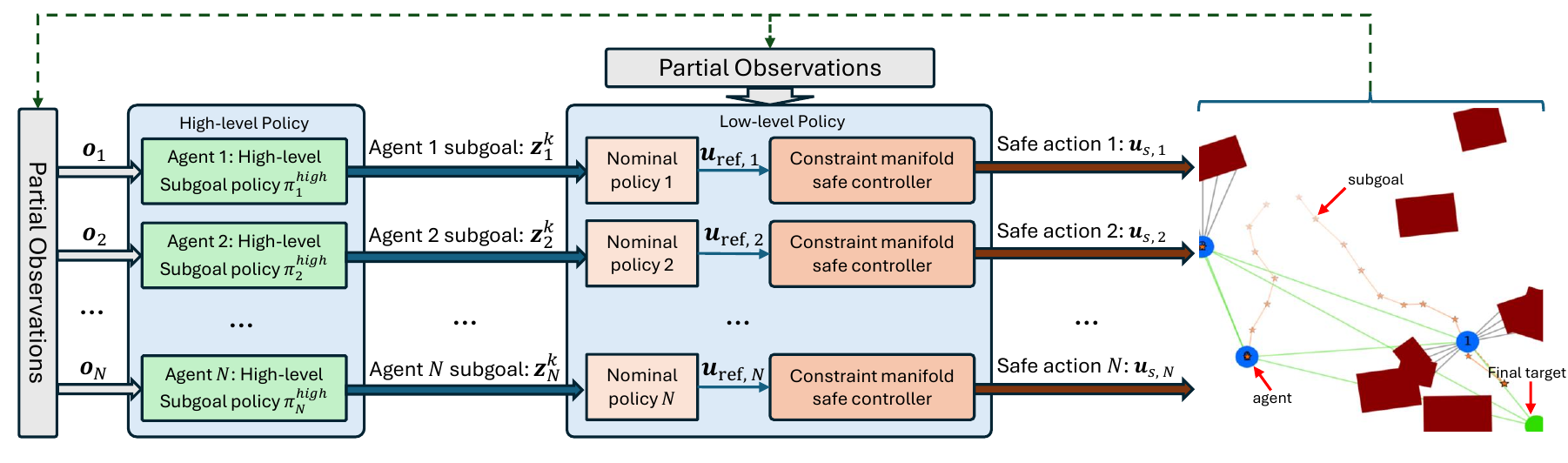} 
    \vspace{-2.0em}
    \caption{Overview of the hierarchical safe MARL framework during 
decentralized execution. Each agent's high-level policy generates a 
subgoal $\mathbf{z}_i^k$, which is converted into a nominal action 
$\mathbf{u}_{\text{ref},i}$ and then projected onto the constraint manifold to produce a safe action $\mathbf{u}_{s,i}$.}

    \label{fig:subgoal_idea} 
    \vspace{-1.0em}
\end{figure}

\subsection{High-Level Subgoal Policy Learning}

The high-level planner is responsible for generating temporally extended subgoals that coordinate cooperative behavior across agents. At each decision epoch $k$, each agent $i$ uses its high-level policy $\pi_i^{high}$, conditioned on its observation $o_i^k$, to generate a subgoal $\mathbf{z}_i^k \in \mathbb{R}^d$. The subgoal represents a target position relative to agent $i$'s current state, where $d$ is the spatial dimension of the environment. This subgoal serves as a short-horizon waypoint that the low-level safe controller drives the agent toward over the subsequent $\tau$ environment timesteps.

We adopt a MARL framework under the centralized training with decentralized execution (CTDE) paradigm for high-level coordination. 
The high-level policy is based on MAPPO~\citep{yu2022surprisingmappo} and incorporates graph-based information aggregation inspired by InforMARL~\citep{nayak2023scalable}. 
Each agent observes its own state, as well as the relative states of nearby agents and obstacles within its sensing radius, forming a local interaction graph. We parameterize the high-level policy $\pi_i^{high}$ using a graph neural network (GNN) ~\citep{velivckovic2017graph, nayak2023scalable, zhang2025gcbf+, zhang2025discrete} with attention-based message passing to aggregate information from neighboring agents and obstacles.
This design ensures permutation invariance and allows the policy to generalize across varying agents and obstacles sizes.
Following the CSMDP formulation with fixed epoch duration $\tau$, the high-level objective is
\begin{equation}
    J(\boldsymbol{\pi}^{high}) = \mathbb{E}_{\boldsymbol{\pi}^{high}}\!\left[
        \sum_{k=0}^{\infty} \gamma^{k\tau}\, r(\mathbf{s}^k, \mathbf{z}^k)
    \right].
\end{equation}

The semi-MDP Bellman residual with effective per-epoch discount $\gamma^\tau$ is 
$\delta^k = r(\mathbf{s}^k, \mathbf{z}^k) + \gamma^{\tau} V_\phi(\mathbf{s}^{k+1}) - V_\phi(\mathbf{s}^k)$, 
and advantages are estimated via Generalized Advantage Estimation: $\hat{A}^k = \sum_{j=0}^{\infty} (\gamma^{\tau} \lambda_{\mathrm{GAE}})^j \, \delta^{k+j}$, where $V_\phi$ is the centralized value function and $\lambda_{\mathrm{GAE}} \in (0,1)$ is the GAE trace-decay parameter.


\subsection{Low-Level Safe Controller}
\label{sec:low_level}

The low-level controller executes the subgoals generated by the high-level policy while strictly satisfying safety constraints at every timestep. Building on the constraint manifold method reviewed in Section~\ref{sec:single_agent_manifold}, the safe controller is applied independently to each agent, exploiting the natural decoupling of control inputs across agents. Given a subgoal $\mathbf{z}_i^k$ from the high-level policy, the low-level controller first computes a nominal reference action $\mathbf{u}_{\text{ref},i} = \pi_{\text{ref}}(\mathbf{s}_i, \mathbf{z}_i^k)$ that drives agent $i$ toward the subgoal. This nominal action does not account for safety constraints; the constraint manifold controller then projects $\mathbf{u}_{\text{ref},i}$ onto the tangent space of the local constraint manifold to produce a safe control input $\mathbf{u}_{s,i}$ for each agent $i$.

\paragraph{Per-agent constraint formulation.}

For each agent $i$, pairwise collision-avoidance constraints are defined with the $L$ nearest entities (agents and obstacles). Each such neighbor $j$ instantiates a per-agent constraint $h_i^{(\ell)}$ from the CSMDP formulation (with $\ell \equiv j$):
\begin{equation}
    h_{ij} = (r_{ij}^{\text{safe}} + \delta_{ij})^2 
    - \|\mathbf{p}_i - \mathbf{p}_j\|^2 \leq 0,
\end{equation}
where $\mathbf{p}_i, \mathbf{p}_j \in \mathbb{R}^d$ are positions and $r_{ij}^{\text{safe}}$ is the safety radius. The velocity-dependent margin $\delta_{ij} = d_0 + v_{\text{approach}}^2 / (2a_{ij})$ accounts for braking distance, where $v_{\text{approach}}$ is the closing speed along the collision axis and $a_{ij}$ is the effective deceleration capacity. For agent-agent pairs, symmetry doubles the collision radius and braking capacity ($r_{ij} = 2r$, $a_{ij} = 2a_{\max}$); agent-obstacle pairs use single-agent values.

\paragraph{Safe control via tangent space projection.} Each agent $i$ independently projects its nominal action $\mathbf{u}_{\text{ref},i}$ onto the tangent space of its local constraint manifold $\mathcal{M}_i$, following the constraint manifold safe controller in~\eqref{eq:atacom}:
\begin{equation}
    \mathbf{u}_{s,i} = \underbrace{-J_{u,i}^\dagger \boldsymbol{\psi}_i}_{\substack{\text{drift} \\ \text{compensation}}}\ \underbrace{- \lambda J_{u,i}^\dagger c_i}_{\substack{\text{manifold} \\ \text{contraction}}}\ + \underbrace{B_{u,i}\, \mathbf{u}_{\text{ref},i}}_{\substack{\text{tangential} \\ \text{tracking}}},
\end{equation}
where $J_{u,i}$, $\boldsymbol{\psi}_i$, $c_i$, and $B_{u,i}$ are computed locally from agent $i$'s state and the states of its neighbors. The resulting safe control $\mathbf{u}_{s,i}$ is guaranteed to keep agent $i$'s augmented state on $\mathcal{M}_i$, ensuring constraint satisfaction at every timestep.


\section{Theoretical Guarantees}

\subsection{Safety Guarantee}

\begin{theorem}[Compositional Safety]
\label{thm:compositional_safety}
Consider $N$ agents with decoupled control-affine dynamics 
$\dot{\mathbf{s}}_i = f_i(\mathbf{s}_i) + G_i(\mathbf{s}_i)\mathbf{u}_i$, 
each applying the constraint manifold controller~\eqref{eq:atacom} 
on its local constraint set $\mathbf{h}_i(\mathbf{s}_i, \mathbf{s}_{\mathcal{N}_i}) \leq \mathbf{0}$, 
where $\mathcal{N}_i = \mathcal{E}_i \cup \mathcal{A}_i$ denotes neighboring obstacles and agents. 
Under the standard assumptions (Assumptions~\ref{asmp:app_compact}--\ref{asmp:app_rank} in Appendix~\ref{app:safety_guarantees}), 
if all agents start on their local constraint manifolds, the joint trajectory remains in the global safe set for all time.

\end{theorem}

\begin{proof}[Proof sketch (full proof in Appendix~\ref{app:multi_agent_safety_proof})]
Each agent $i$ defines a constraint residual
$\mathbf{c}_i = \mathbf{h}_i(\mathbf{s}_i, \mathbf{s}_{\mathcal{N}_i}) + \boldsymbol{\mu}_i$ and a Lyapunov
function $V_i = \tfrac{1}{2}\|\mathbf{c}_i\|^2$.
Expanding $\dot{\mathbf{c}}_i$ via~\eqref{eq:linear_equation} and~\eqref{eq:linear_system} gives
\begin{equation}
    \dot{\mathbf{c}}_i
    = \frac{\partial \mathbf{h}_i}{\partial \mathbf{s}_i}\dot{\mathbf{s}}_i
    + \underbrace{
        \sum_{e \in \mathcal{E}_i}
        \frac{\partial \mathbf{h}_i}{\partial \mathbf{s}_e}\dot{\mathbf{s}}_e
      }_{\text{obstacle motion}}
    + \underbrace{
        \sum_{j \in \mathcal{A}_i}
        \frac{\partial \mathbf{h}_i}{\partial \mathbf{s}_j}\dot{\mathbf{s}}_j
      }_{\text{other agents}}
    + \dot{\boldsymbol{\mu}}_i
    = \boldsymbol{\psi}_i + J_{u,i}
    \begin{bmatrix} \mathbf{u}_{s,i} \\ \mathbf{u}_{\mu,i} \end{bmatrix},
    \label{eq:cdot_multi}
\end{equation}
where the drift $\boldsymbol{\psi}_i$ absorbs all terms not controlled by agent~$i$: its autonomous dynamics, obstacles, and the states of other agents.
In the single-agent setting, only the obstacle term remains and the multi-agent extension adds the inter-agent term, but crucially both enter only the drift $\boldsymbol{\psi}_i$.

Substituting~\eqref{eq:cdot_multi} into $\dot{V}_i = \mathbf{c}_i^\top \dot{\mathbf{c}}_i$ 
and expanding (see~\eqref{eq:app_vdot_expand_multi}), 
the drift term $\mathbf{c}_i^\top(\boldsymbol{\psi}_i - J_{u,i}J_{u,i}^\dagger \boldsymbol{\psi}_i)$ vanishes 
since Assumption~\ref{asmp:app_rank} implies $J_{u,i}J_{u,i}^\dagger = I$. 
The null-space term $\mathbf{c}_i^\top J_{u,i} B_{u,i}\mathbf{u}_{\mathrm{ref},i} = 0$ by construction. 
The remaining contraction term yields
\begin{equation}
    \dot{V}_i = -\lambda\,\mathbf{c}_i^\top J_{u,i}J_{u,i}^\dagger\,\mathbf{c}_i \leq 0.
\end{equation}
Since $V_i(0) = 0$ and $\dot{V}_i \leq 0$, we have $V_i(t) = 0$ for
all $t$, so every agent remains on its local constraint manifold and
satisfies $\mathbf{h}_i(t) \leq \mathbf{0}$
(Remark~\ref{rmk:app_projection}).
Because every pairwise inter-agent constraint $h_{ij}$ is enforced by
at least one agent's local guarantee, intersecting all local safe sets
recovers global safety.
\end{proof}


\subsection{Reduction to Stationary MARL}
\label{sec:reduction}
 
A well-known challenge in hierarchical RL is that the high-level transition kernel $P_{high}$ depends on the evolving low-level policy, causing non-stationarity \citep{sutton1999between, nachum2018data}. Consequently, the induced high-level process violates the stationarity assumption of Markov decision processes, resulting in non-stationary transition dynamics, biased value estimation, and lower training stability. In our framework, this problem does not arise due to the following property:
 
\begin{proposition}[Stationarity by Construction]
\label{prop:stationarity}
Given that the low-level policy $\pi_{low}$ contains no learnable parameters, the induced high-level transition kernel $P_{high}(s' \mid s, z)$ and reward function $r(s, z)$ are time-invariant and uniquely determined by the environment dynamics.
\end{proposition}
\begin{proof}
The high-level dynamics $P_{high}$ depend only on the environment dynamics $P$, the nominal policy $\pi_{\mathrm{nom}}$, and the constraint manifold projection $\Pi_{\mathcal{M}}$. As all these components are stationary and independent of the high-level policy's learning process, the transition kernel remains constant throughout training.
\end{proof}
Therefore, our framework effectively reduces the hierarchical safe MARL problem to solving an unconstrained and stationary high-level MDP. This reduction provides significant advantages for training stability. By offloading safety to a fixed lower level, we eliminate the non-stationarity and distributional shift typical of hierarchical RL. Consequently, the high-level policy interacts with a consistent environment, allowing standard MARL algorithms like MAPPO \citep{yu2022surprisingmappo}. These algorithms inherit their original convergence guarantees, ensuring stable optimization and stationary throughout the entire learning process.

\section{Experiments}

In this section, we provide a comprehensive evaluation and analysis of our \textbf{HMM} across four key dimensions: overall performance, safety throughout the training process, computational efficiency, and generalization capability. By comparing our approach with several baselines, we demonstrate its superior effectiveness in achieving better performance, strictly maintaining safety constraints even during the early stages of learning, and exhibiting stronger generalization capabilities compared to learning-based baselines.

\subsection{Settings}
\textbf{Environments.} We evaluate our algorithm across a diverse set of tasks. Specifically, we conduct experiments in LiDAR-based environments \citep{keyumarsi2023lidar,zhang2025discrete,zhang2025gcbf+}, including both 2D environments (LidarSpread, LidarTarget, LidarLine, LidarBicycle) and 3D environments (LinearDrone, CrazyFlie) (See Figures~\ref{fig:showcase}). To comprehensively assess performance against baselines, we adopt two metrics: safety and success rate. The safety rate measures the ratio of agents that remain safe throughout the entire trajectory. 
The success rate is defined as the ratio of agents that both remain safe and reach their goals.


\textbf{Baselines.} We compare against DGPPO~\citep{zhang2025discrete}, InforMARL~\citep{nayak2023scalable}, and MAPPO-Lagrangian~\citep{gu2021multi, gu2023safe}. 
In InforMARL, constraint violations are incorporated via a penalty term $\beta \cdot \max\{0, \max_m h^{(m)}\}$ with varying penalty weights (denoted \texttt{Penalty($\beta$)} in figures for $\beta \in \{0.02, 0.1, 0.5\}$), and a scheduled variant (\texttt{Schedule}) where $\beta$ is gradually increased during training is also evaluated. 
For MAPPO-Lagrangian, a GNN backbone is used and different Lagrange multiplier learning rates are considered (\texttt{Lagr(1)}, \texttt{Lagr(5)} for multiplier rates $1$ and $5$, and \texttt{Lagr(lr)} for increasing the learning rate to ${0.1}$). Our method is denoted \textbf{HMM} in all figures.

\begin{figure*}[h]
\centering
\begin{subfigure}[b]{0.31\textwidth}
    \includegraphics[width=\textwidth]{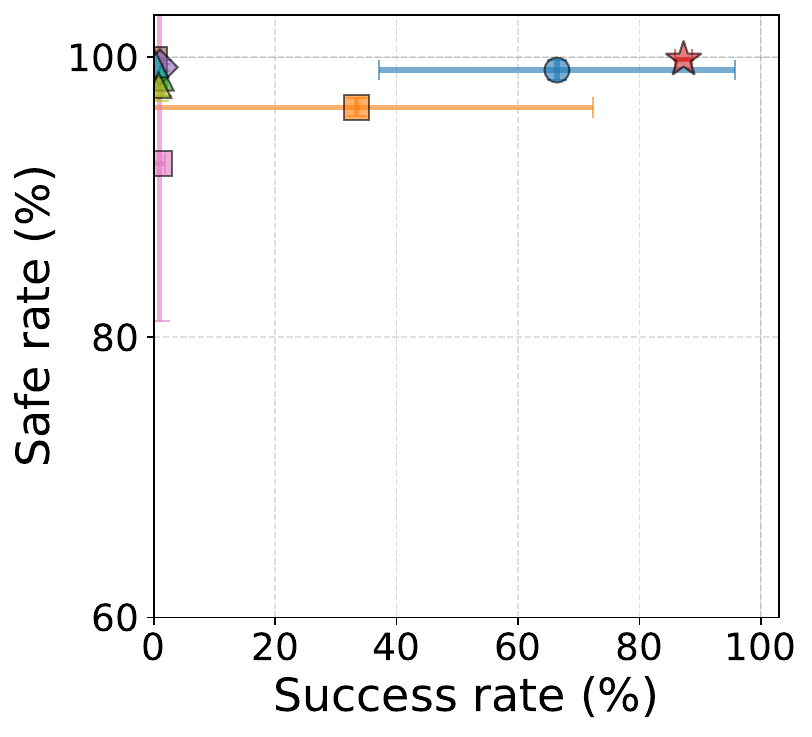}
    \vspace{-1.5em}
    \caption{LidarSpread}
    \label{fig:spread}
\end{subfigure}
\hfill
\begin{subfigure}[b]{0.31\textwidth}
    \includegraphics[width=\textwidth]{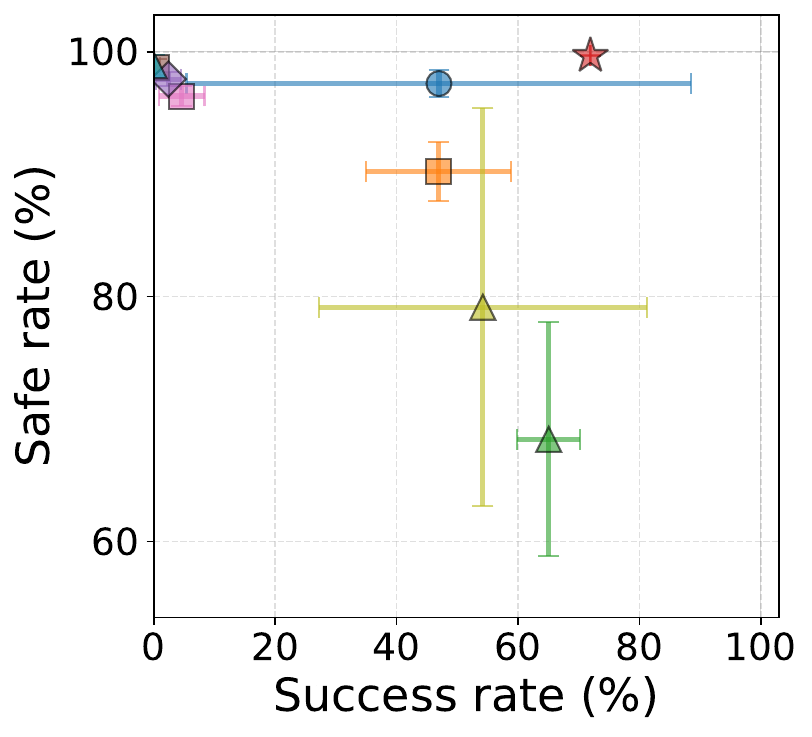}
    \vspace{-1.5em}
    \caption{LidarTarget}
    \label{fig:target}
\end{subfigure}
\hfill
\begin{subfigure}[b]{0.31\textwidth}
    \includegraphics[width=\textwidth]{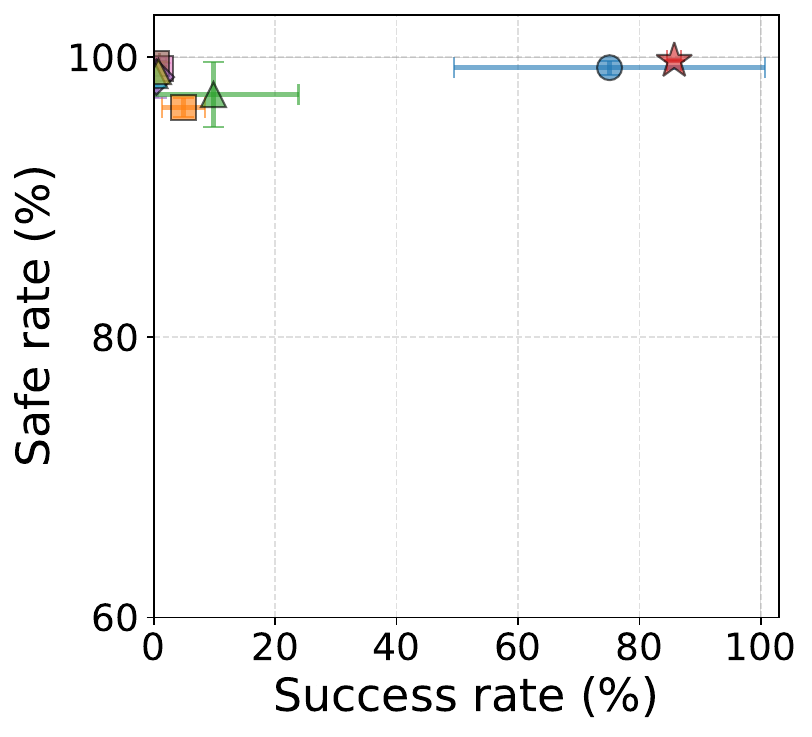}
    \vspace{-1.5em}
    \caption{LidarLine}
    \label{fig:line}
\end{subfigure}


\begin{subfigure}[b]{0.31\textwidth}
    \includegraphics[width=\textwidth]{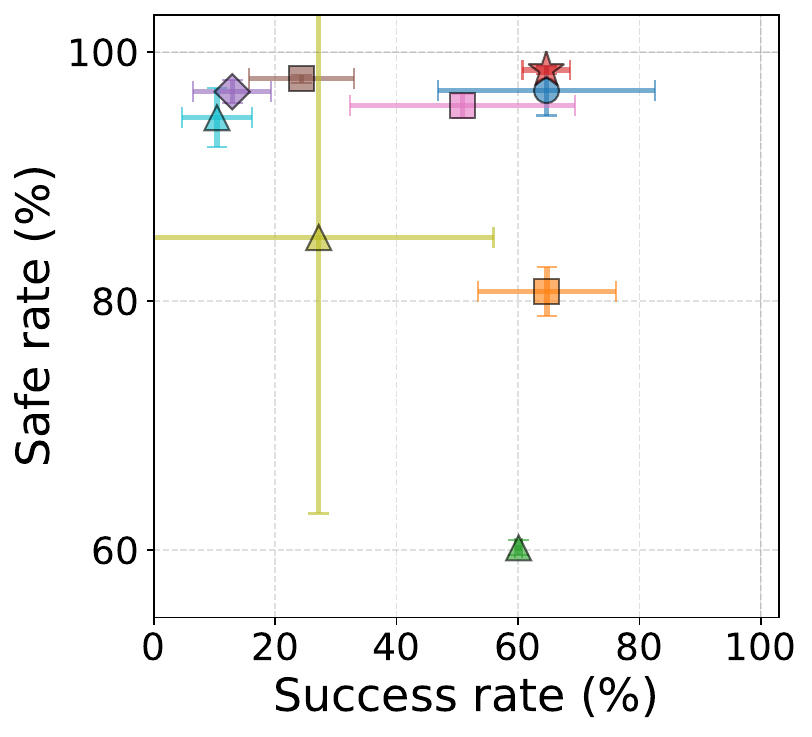}
        \vspace{-1.5em}
    \caption{LidarBicycle}
    \label{fig:bicycle}
\end{subfigure}
\hfill
\begin{subfigure}[b]{0.31\textwidth}
    \includegraphics[width=\textwidth]{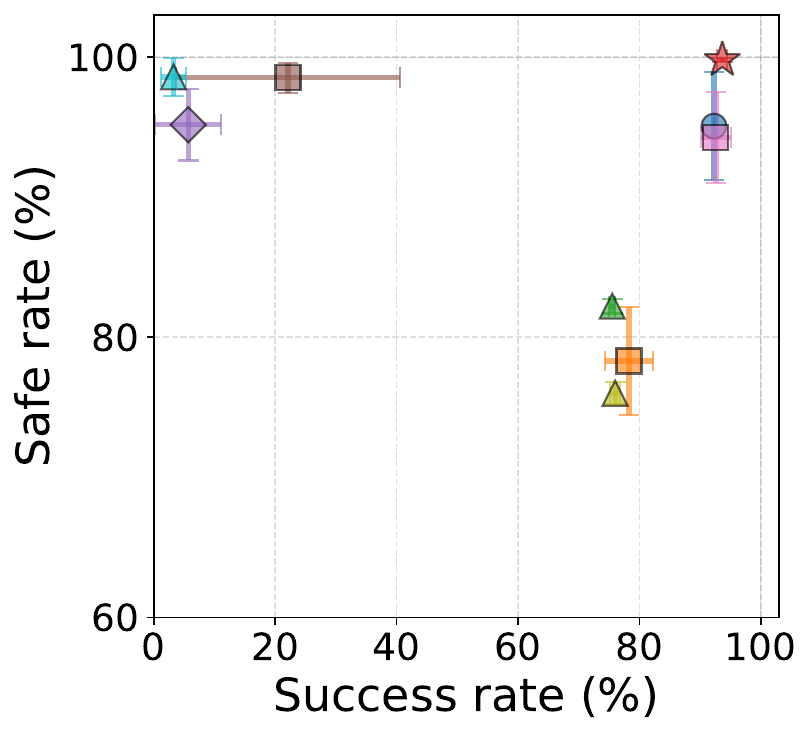}
        \vspace{-1.5em}
    \caption{LinearDrone}
    \label{fig:lineardrone}
\end{subfigure}
\hfill
\begin{subfigure}[b]{0.31\textwidth}
    \includegraphics[width=\textwidth]{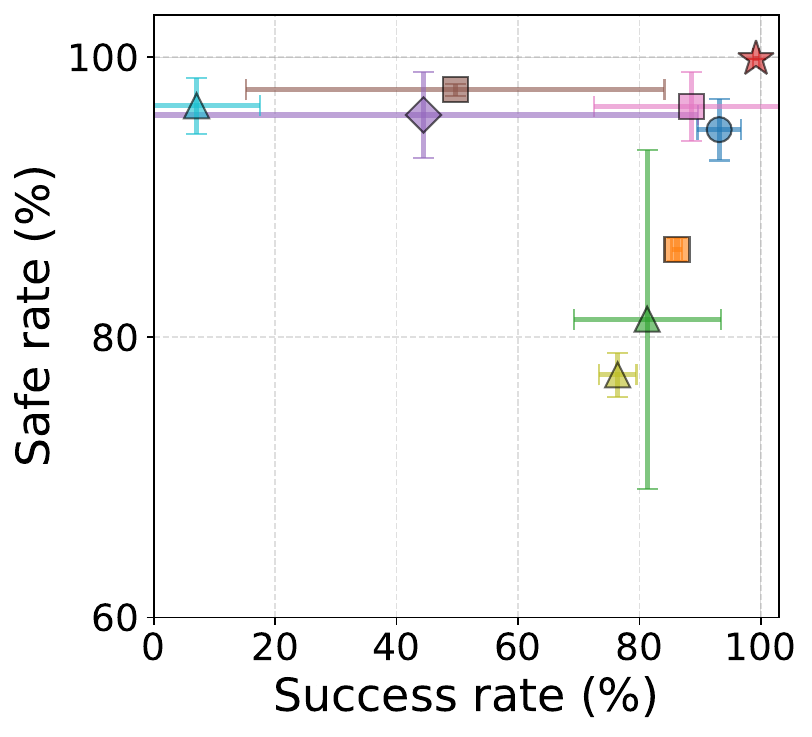}
        \vspace{-1.5em}
    \caption{CrazyFlie}
    \label{fig:crazyflie}
\end{subfigure}

\includegraphics[width=1\textwidth]{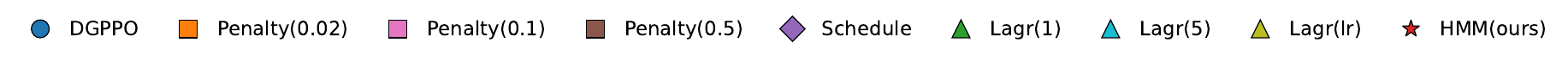}

\vspace{-0.5em}
\caption{Safe rate and success rate (mean $\pm$ std) across six environments.}
\vspace{-1.5em}
\label{fig:best_mean_std}
\end{figure*}

\subsection{Main Results}

\textbf{Overall Performance and Stability.}
We evaluate the performance of different algorithms and visualize their results in terms of safety rate and success rate in a two-dimensional plot. For each method, we report both the mean and standard deviation (std) in the figure.
As shown in Figure~\ref{fig:best_mean_std}, our algorithm consistently outperforms all other baselines in both safety rate and success rate. Notably, our method achieves a highly stable safety rate across all environments, remaining close to 100\% throughout. Moreover, the low std reported in Figure~\ref{fig:best_mean_std} indicates that our method exhibits consistently low variance across runs, demonstrating strong training stability. This suggests that our algorithm is not only effective but also robust during training.

\textbf{Persistent Safety throughout Training.} We present training curves of the safety rate during the learning process. As shown in Figure~\ref{fig:training_safe_curves}, other learning-based baselines are not able to guarantee safety at the early stage of training when the neural networks have not yet converged. In contrast, our \textbf{HMM} consistently maintains safety throughout the entire training process. 
The rare residual violations (<2\%) stem from discrete-time integration, neighbor truncation, and partial observability, which fall outside our continuous-time guarantee (Theorem~\ref{thm:compositional_safety}); see Appendix~\ref{app:violations} for details.

\begin{figure*}[h]
\centering
\begin{subfigure}[b]{0.32\textwidth}
    \includegraphics[width=\textwidth]{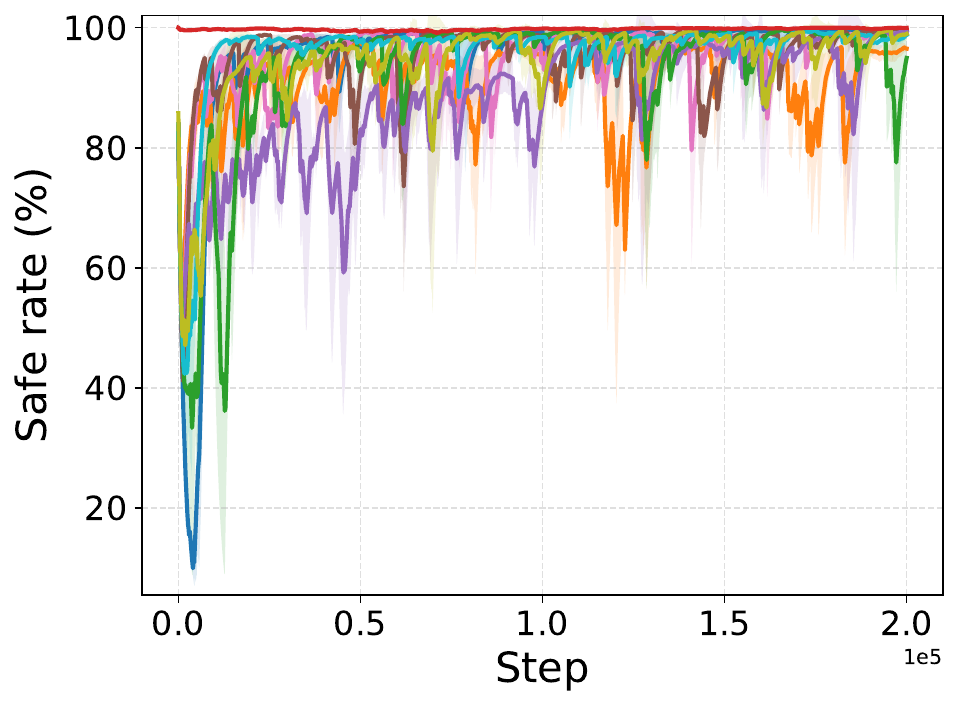}
        \vspace{-1.5em}
    \caption{LidarSpread}
    \label{fig:curve_spread}
\end{subfigure}
\hfill
\begin{subfigure}[b]{0.32\textwidth}
    \includegraphics[width=\textwidth]{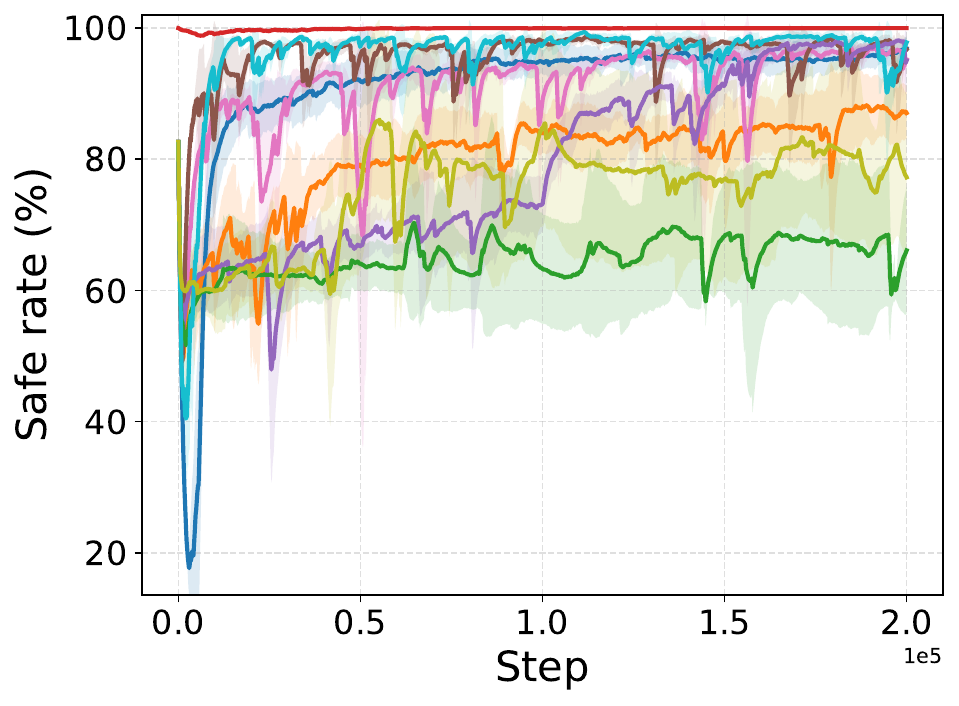}
        \vspace{-1.5em}
    \caption{LidarTarget}
    \label{fig:curve_target}
\end{subfigure}
\hfill
\begin{subfigure}[b]{0.32\textwidth}
    \includegraphics[width=\textwidth]{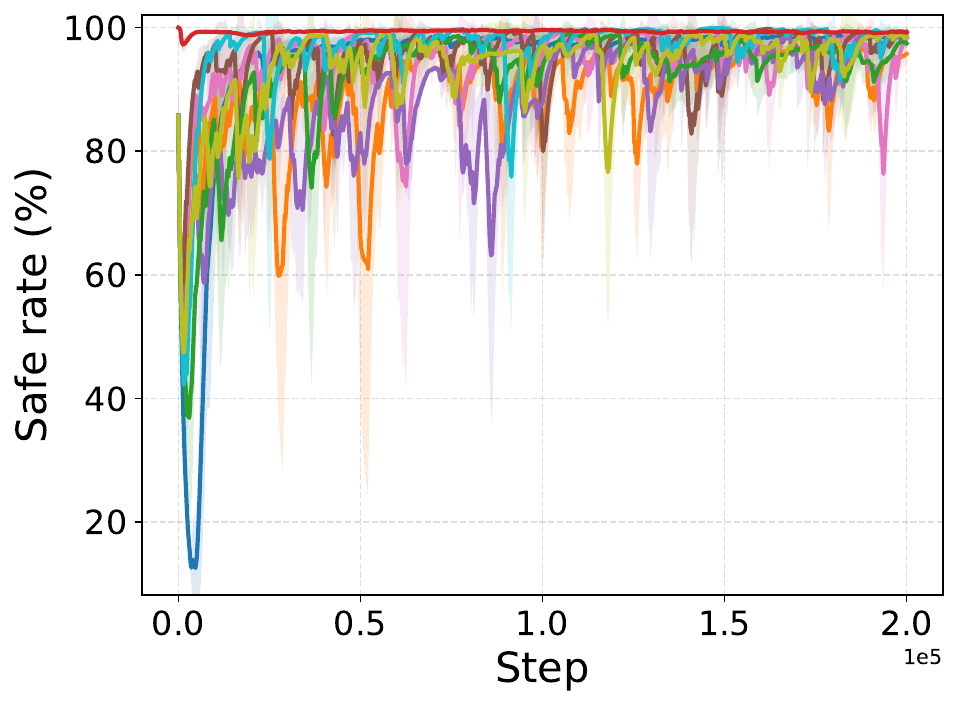}
        \vspace{-1.5em}
    \caption{LidarLine}
    \label{fig:curve_line}
\end{subfigure}

\begin{subfigure}[b]{0.32\textwidth}
    \includegraphics[width=\textwidth]{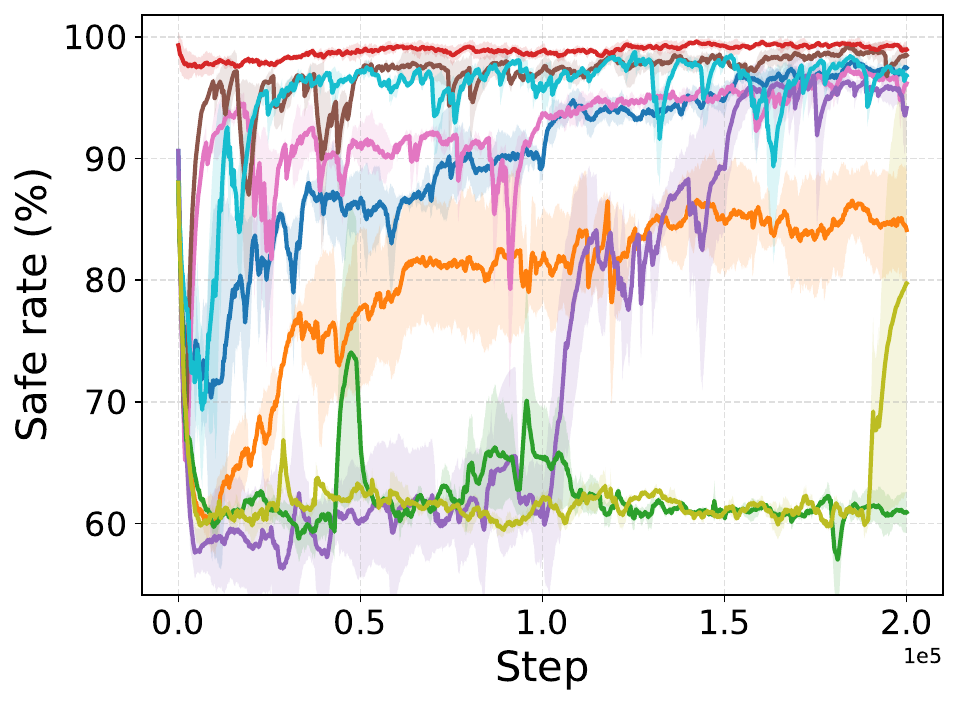}
        \vspace{-1.5em}
    \caption{LidarBicycle}
    \label{fig:curve_bicycle}
\end{subfigure}
\hfill
\begin{subfigure}[b]{0.32\textwidth}
    \includegraphics[width=\textwidth]{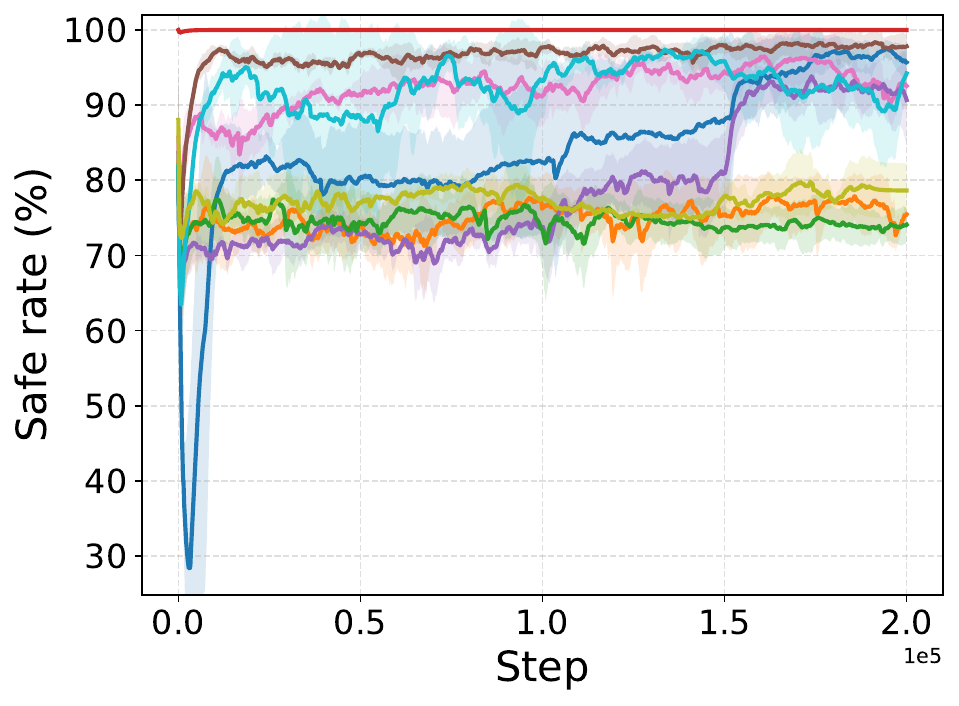}
        \vspace{-1.5em}
    \caption{LinearDrone}
    \label{fig:curve_lineardrone}
\end{subfigure}
\hfill
\begin{subfigure}[b]{0.32\textwidth}
    \includegraphics[width=\textwidth]{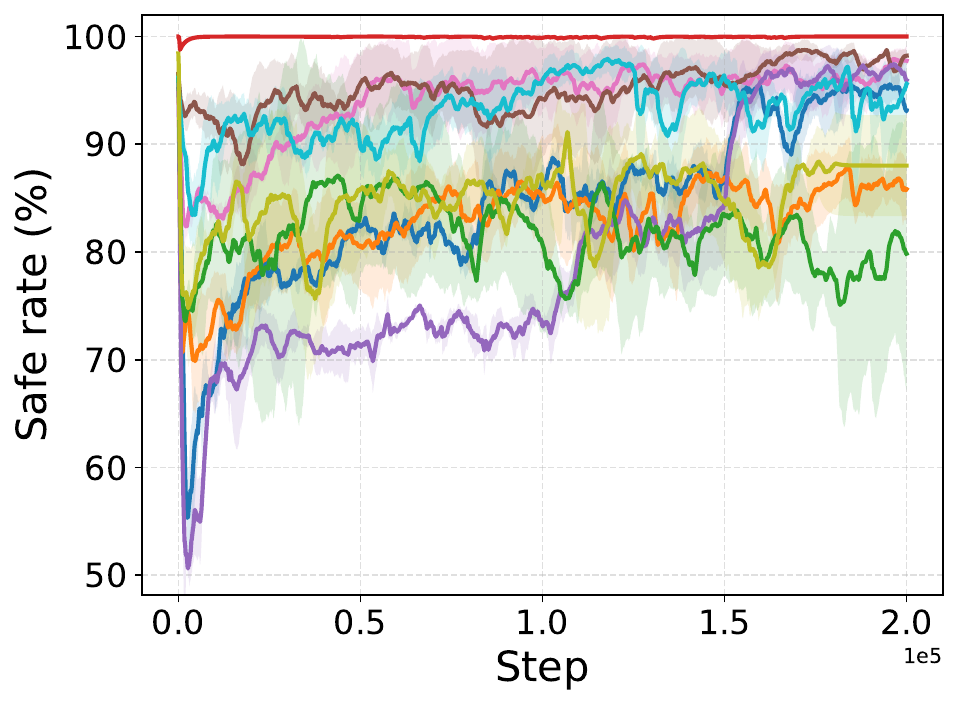}
        \vspace{-1.5em}
    \caption{CrazyFlie}
    \label{fig:curve_crazyflie}
\end{subfigure}

\includegraphics[width=1\textwidth]{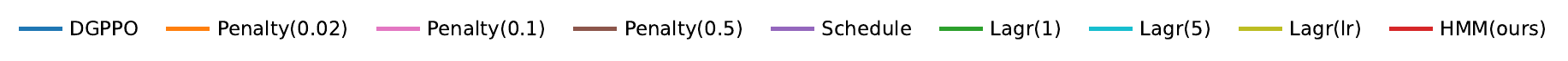}

\caption{Training safe rate curves across six environments.}
\vspace{-1.0em}
\label{fig:training_safe_curves}
\end{figure*}

\textbf{Computational Efficiency.} We compare the training efficiency of our constraint manifold controller against the QP-based CBF baseline on a single NVIDIA H200 GPU. Benefiting from the closed-form pseudoinverse projection, our method achieves a significant speedup in training throughput. Specifically, based on the same training setting, while the CBF-QP baseline requires over 5 days to complete 200k steps, our \textbf{HMM} finishes the same training process within 8 to 10 hours. A detailed breakdown of the per-iteration cost and total training time is provided in Appendix~\ref{app:Computational_efficiency}.

\textbf{Ablation Study.}
To validate the contribution of each component in our framework, we compare three variants: (1) Nominal only, which uses a simple proportional controller to drive agents toward their goals without any safety mechanism; (2) Nominal + Manifold, which applies the constraint manifold safe controller on top of the nominal controller but without learned high-level subgoal planning; and (3) \textbf{HMM} (ours), the full framework with both learned high-level subgoals and the low-level manifold controller. As shown in Table~\ref{tab:manifold_only_results}, the nominal-only baseline achieves poor safety across all environments, confirming that naive goal-tracking is insufficient in multi-agent settings. Adding the manifold controller significantly improves safety (above 96\% in all cases), but without high-level coordination, agents frequently get stuck or take suboptimal paths, resulting in limited success rates. Our \textbf{HMM} framework achieves the highest performance on both metrics, demonstrating that the learned high-level planner is essential for effective coordination while the low-level manifold controller provides the safety foundation.  Notably, the addition of learned high-level coordination leads to dramatic improvements in success rate, boosting performance by over 35\% points in LidarSpread and LidarLine, and by more than 50\% points in LinearDrone. This highlights the importance of high-level planning for coordinating agents beyond merely satisfying safety constraints.

\begin{table}[h]
\centering
\caption{Comparison of nominal-only, nominal+manifold, and our \textbf{HMM} method across six environments. All values are percentages (mean $\pm$ std across seeds, where applicable). Bold marks the best safe rate and success rate per environment.}
\label{tab:manifold_only_results}
\begin{tabular}{llcc}
\toprule
\textbf{Environment} & \textbf{Method} & \textbf{Safe rate (\%) $\uparrow$} & \textbf{Success rate (\%) $\uparrow$} \\
\midrule
\multirow{3}{*}{LidarSpread}        & Nominal only         & 32.23                 & 32.17 \\
                               & Nominal + Manifold   & 97.60                 & 49.17 \\
                               & \textbf{HMM (ours)}  & \textbf{99.83\,$\pm$\,0.06} & \textbf{87.29\,$\pm$\,1.35} \\
\midrule
\multirow{3}{*}{LidarTarget}        & Nominal only         & 62.43                 & 59.90 \\
                               & Nominal + Manifold   & 98.77                 & 70.97 \\
                               & \textbf{HMM (ours)}  & \textbf{99.67\,$\pm$\,0.05} & \textbf{71.93\,$\pm$\,0.02} \\
\midrule
\multirow{3}{*}{LidarLine}          & Nominal only         & 33.30                 & 33.13 \\
                               & Nominal + Manifold   & 96.87                 & 48.20 \\
                               & \textbf{HMM (ours)}  & \textbf{99.73\,$\pm$\,0.03} & \textbf{85.75\,$\pm$\,1.15} \\
\midrule
\multirow{3}{*}{LidarBicycle}       & Nominal only         & 60.97                 & 59.17 \\
                               & Nominal + Manifold   & 97.17                 & 54.67 \\
                               & \textbf{HMM (ours)}  & \textbf{98.58\,$\pm$\,0.31} & \textbf{64.67\,$\pm$\,3.93} \\
\midrule
\multirow{3}{*}{LinearDrone}   & Nominal only         & 74.43                 & 40.00 \\
                               & Nominal + Manifold   & 99.57                 & 42.87 \\
                               & \textbf{HMM (ours)}  & \textbf{99.80\,$\pm$\,0.14} & \textbf{93.67\,$\pm$\,0.80} \\
\midrule
\multirow{3}{*}{CrazyFlie}     & Nominal only         & 74.20                 & 74.20 \\
                               & Nominal + Manifold   & 98.63                 & 94.93 \\
                               & \textbf{HMM (ours)}  & \textbf{99.87\,$\pm$\,0.00} & \textbf{99.22\,$\pm$\,0.44} \\
\bottomrule
\end{tabular}
\end{table}

\subsection{Generalization}

All our 2D environments are trained with 3 agents and 3 obstacles, while the 3D environments are trained with 3 agents and 6 obstacles. During evaluation, we increase the number of agents and obstacles while keeping the map size fixed. As shown in Figure~\ref{fig:generalization_main}, when the number of agents increases, all other baselines fail to maintain safety and often break down. This suggests that the other baselines struggle to generalize under distribution shifts, particularly as interaction complexity increases with larger numbers of agents and obstacles. In contrast, our \textbf{HMM} continues to guarantee safety even under increased numbers of agents or obstacles, and in most cases achieves an even higher success rate. Results for the remaining environments are provided in the appendix~\ref{app:add_generalization_results}.

\begin{figure}[h]
    \centering
    \begin{subfigure}[b]{0.24\textwidth}
        \includegraphics[width=\textwidth]{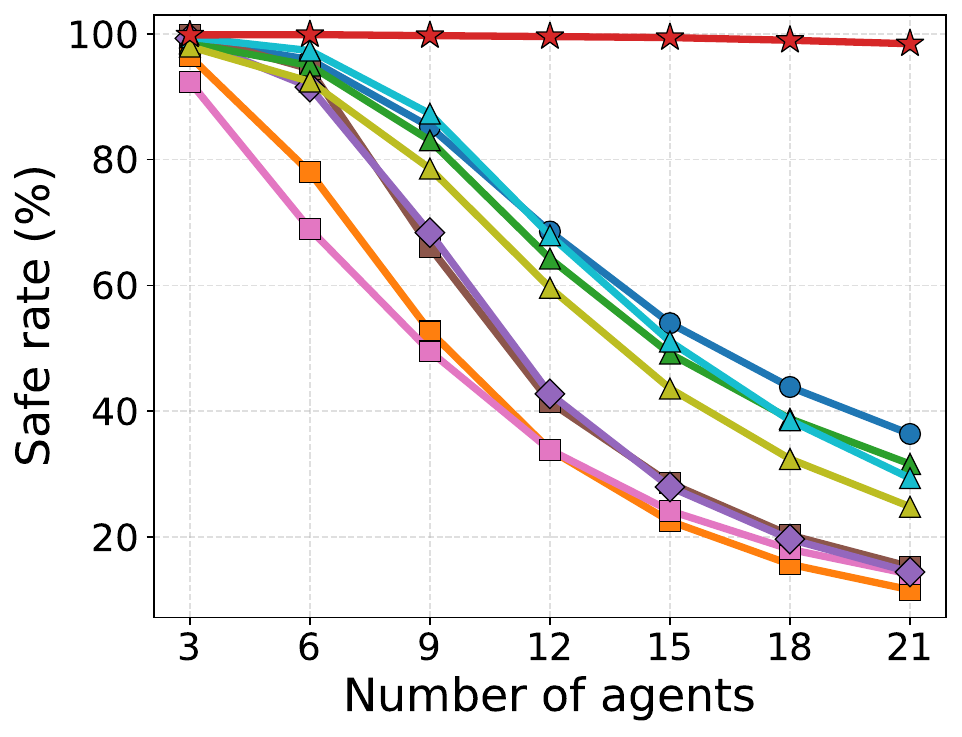}
                \vspace{-1.5em}
        \caption{}
    \end{subfigure}
    \hfill
    \begin{subfigure}[b]{0.24\textwidth}
        \includegraphics[width=\textwidth]{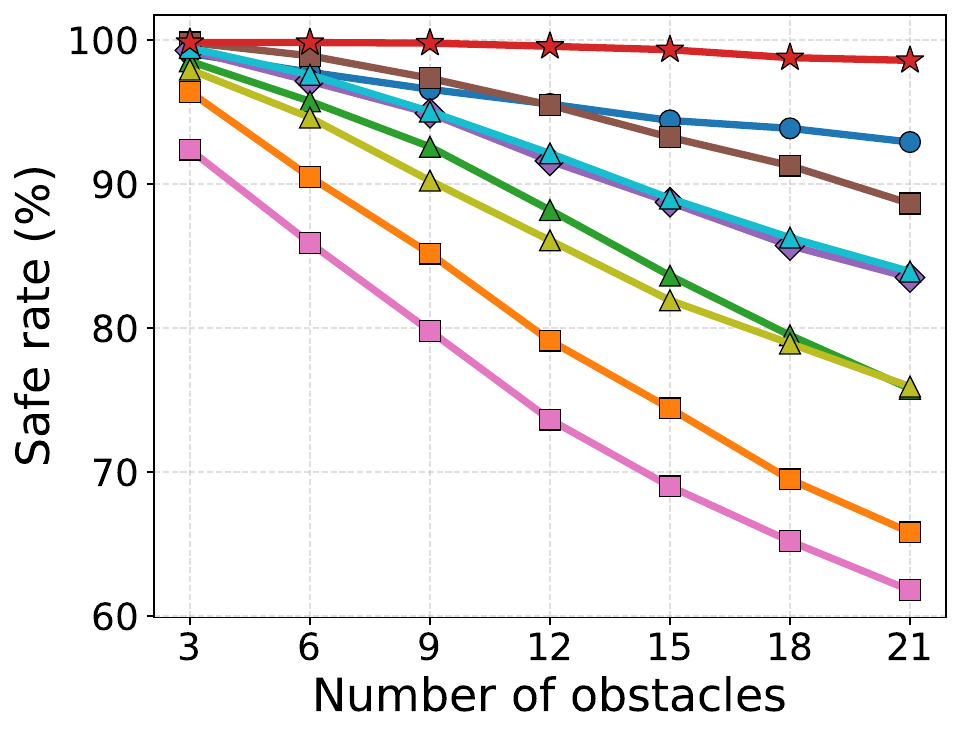}
                \vspace{-1.5em}
        \caption{}
    \end{subfigure}
    \hfill
    \begin{subfigure}[b]{0.24\textwidth}
        \includegraphics[width=\textwidth]{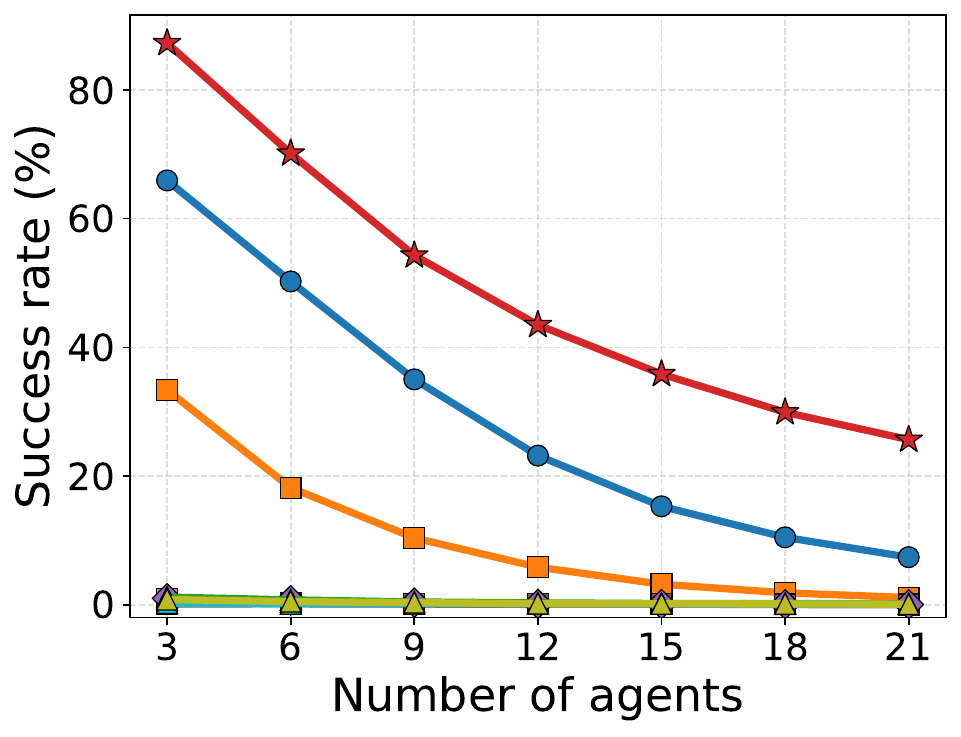}
                \vspace{-1.5em}
        \caption{}
    \end{subfigure}
    \hfill
    \begin{subfigure}[b]{0.24\textwidth}
        \includegraphics[width=\textwidth]{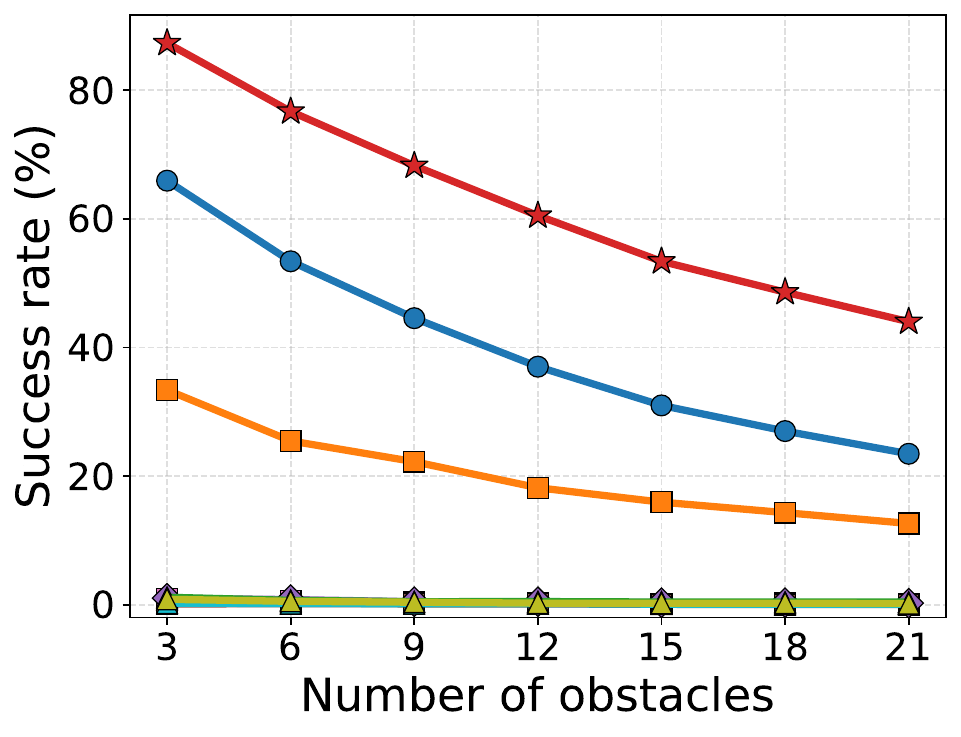}
                \vspace{-1.5em}
        \caption{}
    \end{subfigure}

    \begin{subfigure}[b]{0.24\textwidth}
        \includegraphics[width=\textwidth]{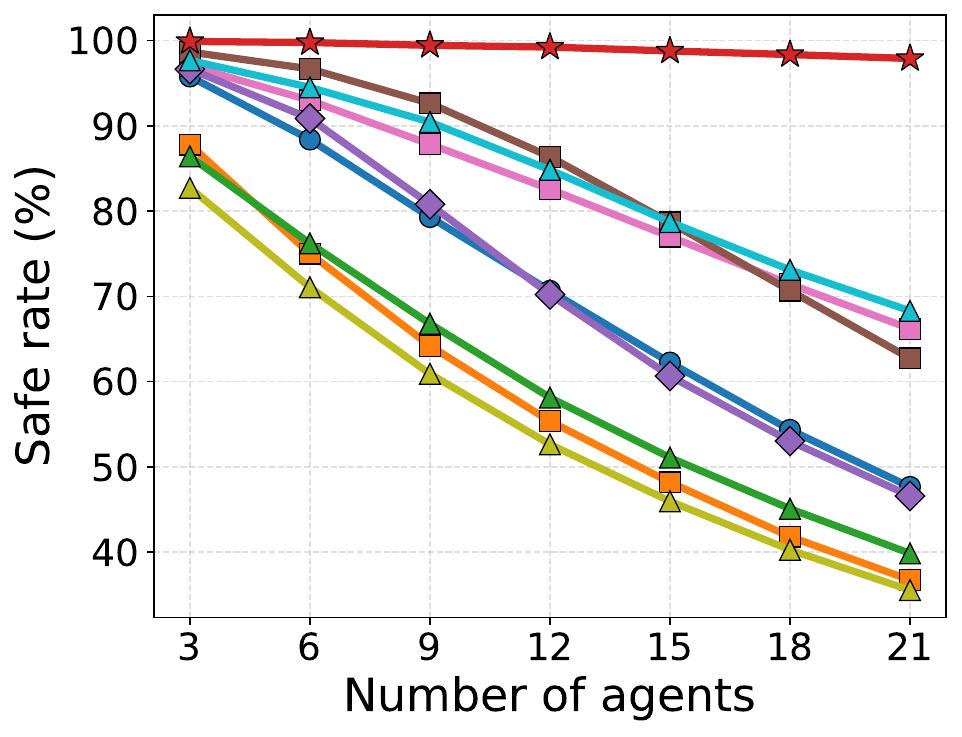}
                \vspace{-1.5em}
        \caption{}
    \end{subfigure}
    \hfill
    \begin{subfigure}[b]{0.24\textwidth}
        \includegraphics[width=\textwidth]{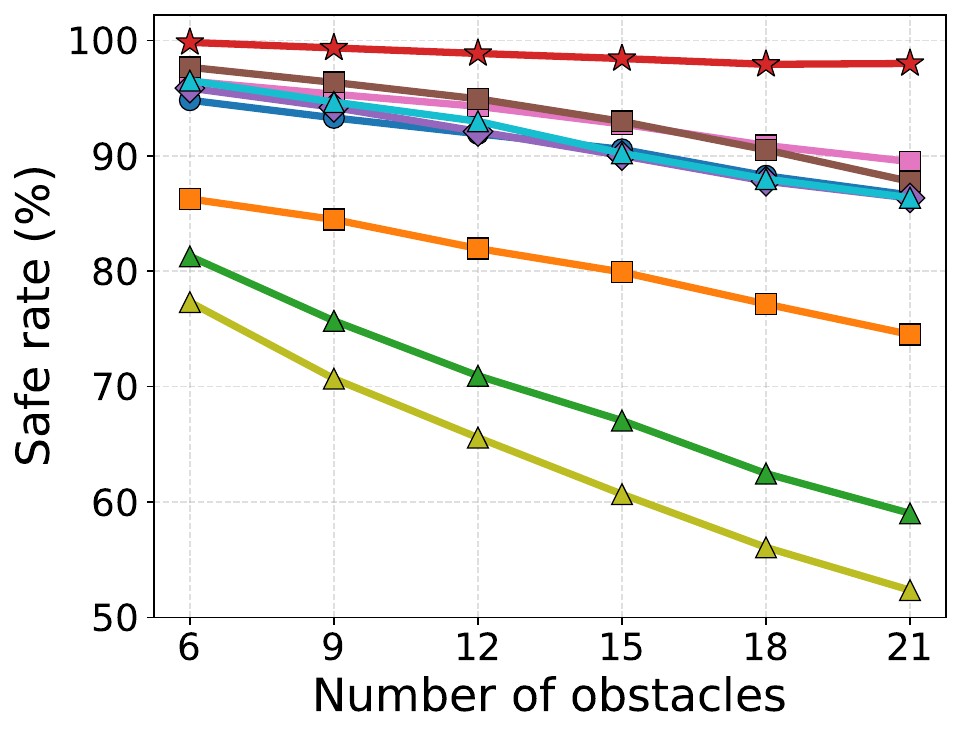}
                \vspace{-1.5em}
        \caption{}
    \end{subfigure}
    \hfill
    \begin{subfigure}[b]{0.24\textwidth}
        \includegraphics[width=\textwidth]{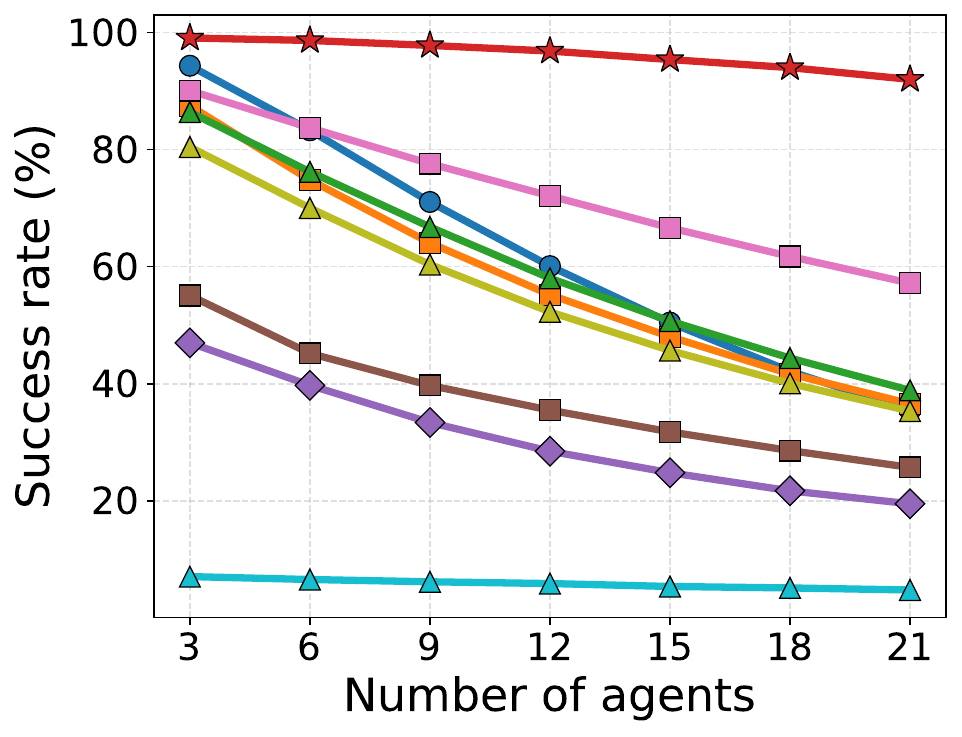}
                \vspace{-1.5em}
        \caption{}
    \end{subfigure}
    \hfill
    \begin{subfigure}[b]{0.24\textwidth}
        \includegraphics[width=\textwidth]{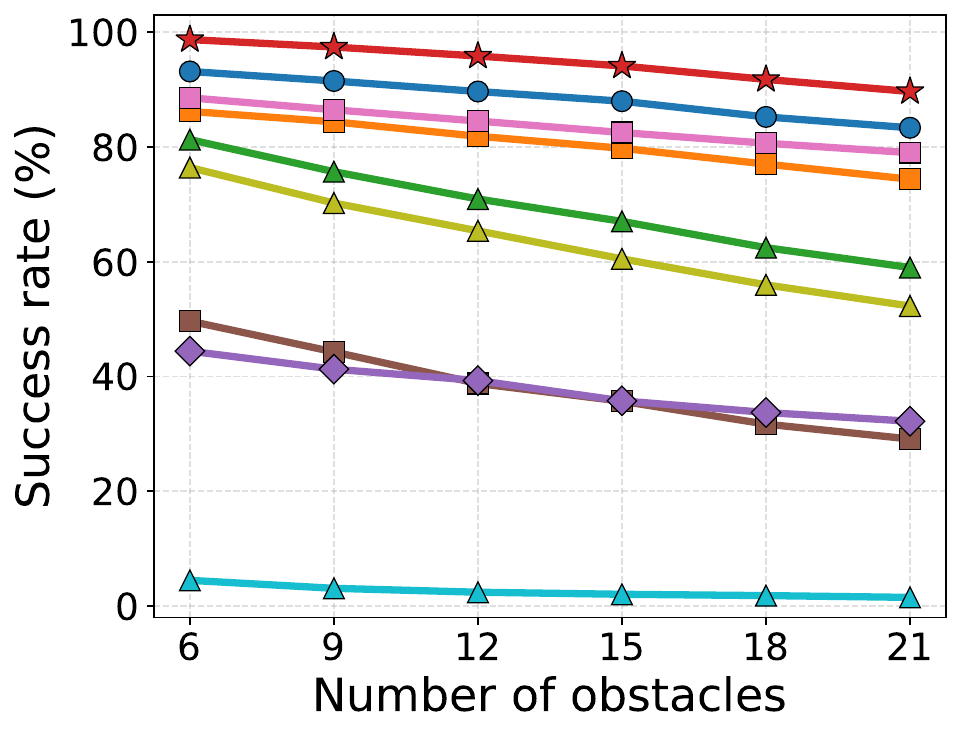}
                \vspace{-1.5em}
        \caption{}
    \end{subfigure}

    \includegraphics[width=1\textwidth]{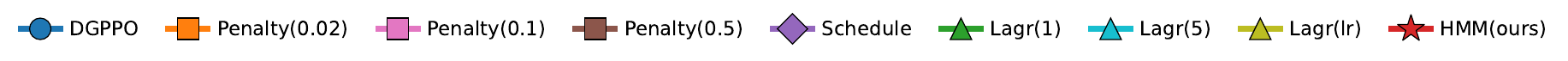}
    
    \caption{Generalization results. LidarSpread (a--d) and CrazyFlie (e--h). Left two columns: safe rate; right two columns: success rate. Models trained with 3 agents and 3/6 obstacles are evaluated with varying numbers of agents and obstacles.}
    \label{fig:generalization_main}
    \vspace{-1.2em}
\end{figure}

\section{Conclusion}

We propose a hierarchical safe multi-agent reinforcement learning framework. The high-level policy adopts a CTDE MARL architecture to generate subgoals, while a fixed low-level constraint-manifold controller projects these subgoals onto the safe action space. We provide theoretical guarantees on both safety and convergence of the proposed approach.
Empirically, our method achieves near-perfect safety rates across multiple multi-agent navigation tasks, while maintaining higher task success rates compared to other baselines.
Furthermore, it demonstrates strong generalization across varying numbers of agents and obstacles.
Overall, our results highlight the effectiveness of integrating learning-based coordination with control-theoretic safety mechanisms, offering a promising direction for developing scalable and reliable multi-agent systems.

\textbf{Limitation}
The low-level policy of our \textbf{HMM} is based on control-theoretic safety mechanisms, which makes it less suitable for tasks involving complex physical interactions, such as contact-rich manipulation or object transportation.
In particular, extending our framework to scenarios where agents must interact with external objects while ensuring their safety remains challenging.

\section*{Acknowledgement}

This work was supported by the Engineering and Physical Sciences Research Council [grant number EP/Y003187/1 and UKRI849].

\bibliography{references}
\bibliographystyle{icml2026}

\newpage
\appendix
\appendix

\section{Environment Descriptions}

We evaluate our method on six multi-agent environments with LiDAR-based obstacle perception, covering a spectrum of dynamics ranging from simple 2D point-mass systems to higher-dimensional non-holonomic and quadrotor platforms.
LidarSpread and LidarTarget adopt 2D double-integrator dynamics (state $[x, y, \dot{x}, \dot{y}]$, action $[a_x, a_y]$), differing in their goal-assignment semantics. In LidarSpread, $N$ agents are required to collectively cover $N$ goal locations (any-to-any matching), whereas in LidarTarget each agent is assigned a fixed goal ($i \leftrightarrow i$ matching). LidarLine extends this setup by replacing discrete goal nodes with two landmarks, between which $N$ goal positions are linearly interpolated, requiring agents to form an evenly spaced line. LidarBicycleTarget retains the target-based assignment but introduces non-holonomic kinematic bicycle dynamics (5D state including heading and forward velocity).

LinearDrone models a 3D quadrotor with linear dynamics, using a 6D state $[x, y, z, \dot{x}, \dot{y}, \dot{z}]$ and 3D acceleration control, with aerodynamic damping along each axis.
CrazyFlie represents a more challenging setting with full nonlinear quadrotor dynamics and realistic physical parameters. Each agent has a 12D state including position, orientation, and velocities.

\begin{figure}[h]
    \centering
    
    \begin{subfigure}[b]{0.3\textwidth}
        \centering
        \includegraphics[width=\textwidth]{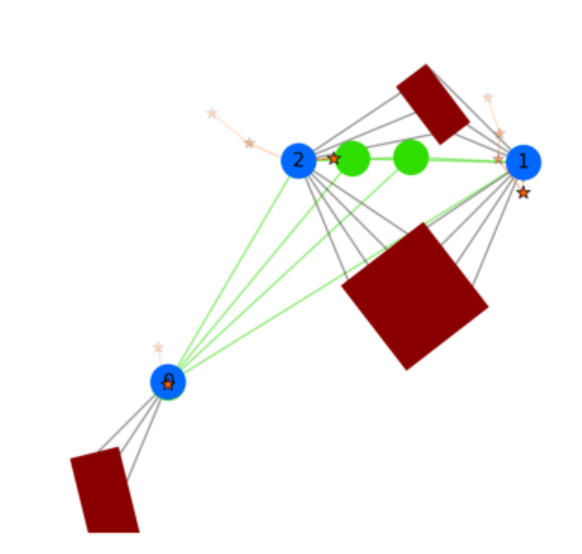}
        \caption{Spread}
    \end{subfigure}
    \hfill
    \begin{subfigure}[b]{0.3\textwidth}
        \centering
        \includegraphics[width=\textwidth]{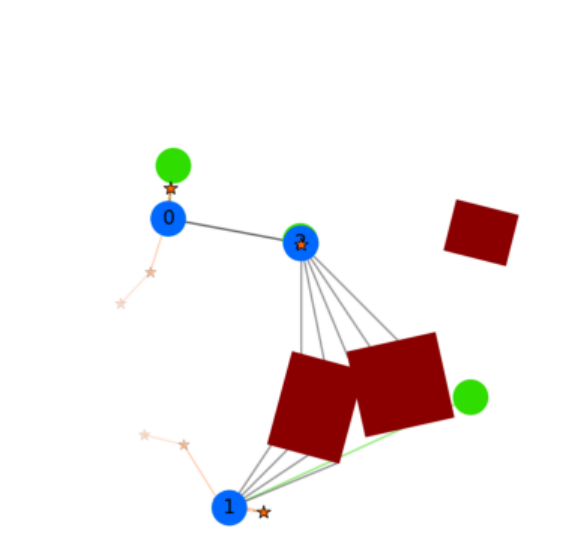}
        \caption{Target}
    \end{subfigure}
    \hfill
    \begin{subfigure}[b]{0.3\textwidth}
        \centering
        \includegraphics[width=\textwidth]{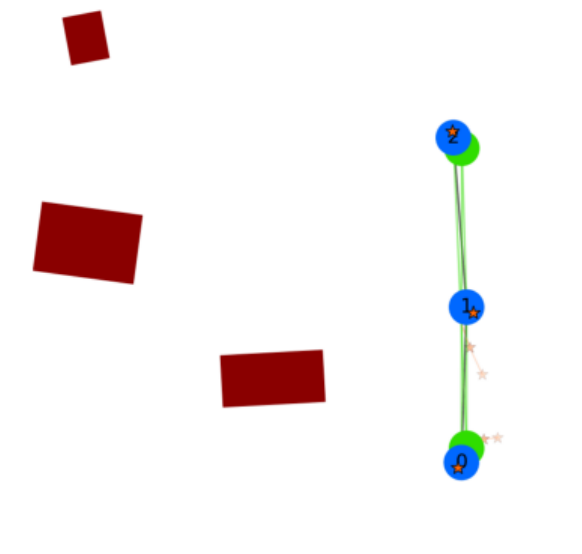}
        \caption{Line}
    \end{subfigure}
    
    \vspace{0.5em}
    
    \begin{subfigure}[b]{0.3\textwidth}
        \centering
        \includegraphics[width=\textwidth]{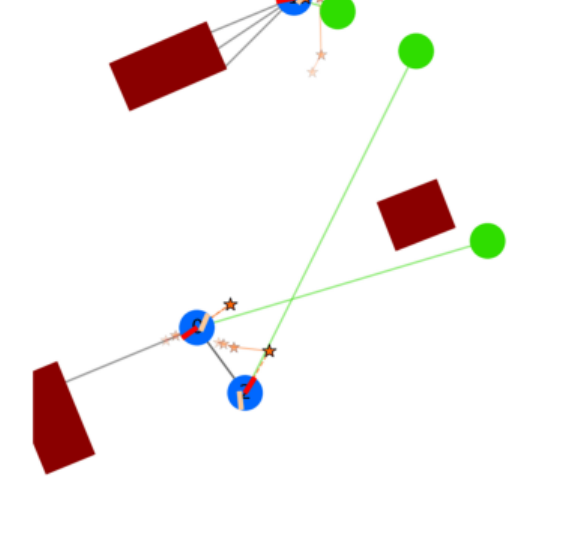}
        \caption{Bicycle}
    \end{subfigure}
    \hfill
    \begin{subfigure}[b]{0.3\textwidth}
        \centering
        \includegraphics[width=\textwidth]{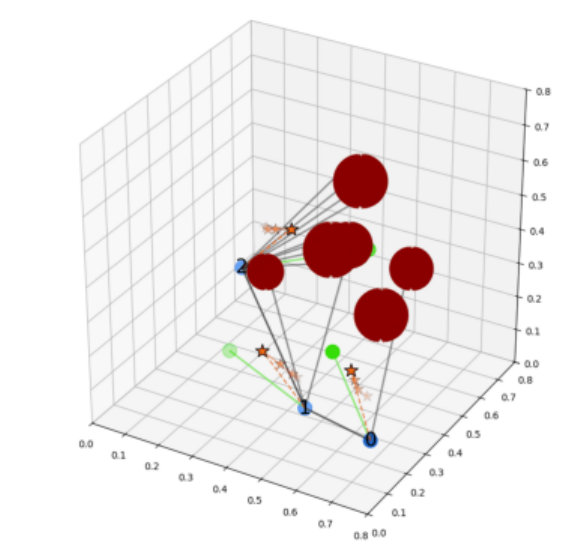}
        \caption{LinearDrone}
    \end{subfigure}
    \hfill
    \begin{subfigure}[b]{0.3\textwidth}
        \centering
        \includegraphics[width=\textwidth]{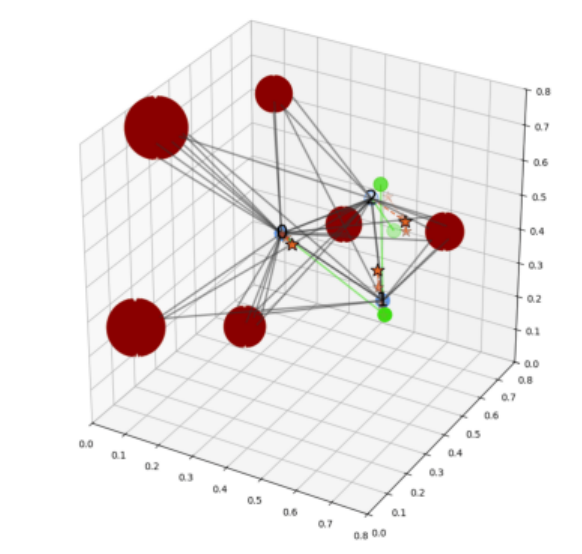}
        \caption{CrazyFlie}
    \end{subfigure}

    \caption{Visualization of different LiDAR-based multi-agent environments.}
    \label{fig:showcase}
\end{figure}

\section{Safety Proofs}
\label{app:safety_proof}

\subsection{Manifold Preliminaries}
\label{app:preliminaries}
We review relevant concepts from manifold theory here; see~\citep{lee2003smooth, boumal2023introduction} for details.

\paragraph{Manifold}

An $n$-dimensional manifold is a topological space that is locally homeomorphic to $\mathbb{R}^n$. Therefore, every point has a neighborhood that is homeomorphic to an open subset of $\mathbb{R}^n$. A differentiable manifold(also called a smooth manifold) is a manifold equipped with an atlas of coordinate charts whose transition maps are all differentiable In this work, we consider differentiable manifolds.

\paragraph{Tangent Space}

Let $\mathcal{M}$ be an $n$-dimensional differentiable manifold and $p \in \mathcal{M}$. The tangent space at $p$, denoted by $\mathrm{T}_p\mathcal{M}$, is an $n$-dimensional vector space of all tangent vectors at $p$, providing a local linear approximation of $\mathcal{M}$ near $p$. The \emph{tangent bundle} $\mathrm{T}\mathcal{M} = \sqcup_{p \in \mathcal{M}} \mathrm{T}_p\mathcal{M}$ is the disjoint union of all tangent spaces over $\mathcal{M}$.

\paragraph{Smooth Map and Rank}
Let $\mathcal{M}$ and $\mathcal{N}$ be smooth manifolds. A smooth map $\Phi: \mathcal{M} \to \mathcal{N}$ induces, at each point $p \in \mathcal{M}$, a linear map $\mathrm{D}\Phi_p: \mathrm{T}_p\mathcal{M} \to \mathrm{T}_{\Phi(p)}\mathcal{N}$, called the differential of $\Phi$ at $p$. The \emph{rank} of $\Phi$ at $p$ is defined as the rank of $\mathrm{D}\Phi_p$.

\paragraph{Embedded Submanifold}
An \emph{embedded submanifold} is a subset $\mathcal{S} \subseteq \mathcal{M}$ that itself forms a manifold, with the \emph{codimension} given by $\dim \mathcal{M} - \dim \mathcal{S}$.

\begin{theorem}[Constant-Rank Level Set Theorem]
    \label{thm:app_level_set}
    Let $\mathcal{M}$ and $\mathcal{N}$ be smooth manifolds, and let $\Phi: \mathcal{M} \rightarrow \mathcal{N}$ be a smooth map with constant rank $r$. Then every level set of $\Phi$ is a properly embedded submanifold of $\mathcal{M}$ with codimension $r$.
\end{theorem}

As a consequence, consider a smooth function $f: \mathcal{E} \rightarrow \mathbb{R}^k$, where $\mathcal{E}$ is a Euclidean space of dimension $d > k$ equipped with an inner product $\langle \cdot, \cdot \rangle$ and the induced norm $\|\cdot\|$. If the differential $Df(\boldsymbol{x})$ has full rank $k$ at every $\boldsymbol{x}$ in the level set
\begin{equation}
    \mathcal{M} = \{\boldsymbol{x} \in \mathcal{E} : f(\boldsymbol{x}) = \boldsymbol{0}\},
\end{equation}
then $\mathcal{M}$ is a smooth \emph{embedded submanifold} of $\mathcal{E}$ with dimension $d - k$. Its tangent space at any point $\boldsymbol{x} \in \mathcal{M}$ is characterized by
\begin{equation}
    \mathrm{T}_{\boldsymbol{x}}\mathcal{M} = \{\boldsymbol{v} \in \mathcal{E} : \langle \mathrm{grad}\, f_i(\boldsymbol{x}),\, \boldsymbol{v} \rangle = 0,\ i \in \{1, \ldots, k\}\},
\end{equation}
which has dimension $\dim \mathrm{T}_{\boldsymbol{x}}\mathcal{M} = d - k$. Intuitively, the tangent space consists of all directions orthogonal to every constraint gradient.

\paragraph{LaSalle's Invariance Principle.}
\label{thm:app_lasalle}
Consider the autonomous system $\dot{\mathbf{x}} = f(\mathbf{x})$ with $f: \mathcal{D} \to \mathbb{R}^n$ locally Lipschitz on a domain $\mathcal{D} \subseteq \mathbb{R}^n$.

\subsection{Single Agent Constraint Manifold Construction}
\label{app:constraint_manifold}

\paragraph{Slack variable augmentation.}
Introduce slack variables $\boldsymbol{\mu} \in [0, +\infty)^K$ and rewrite the inequality constraints as equalities:
\begin{equation}
    \mathbf{c}(\mathbf{x}, \boldsymbol{\mu}) := \mathbf{h}(\mathbf{x}) + \boldsymbol{\mu} = \mathbf{0}.
    \label{eq:app_equality}
\end{equation}
The augmented state is $(\mathbf{x}, \boldsymbol{\mu}) \in \mathcal{D} := \mathcal{X} \times [0, +\infty)^K \subseteq \mathbb{R}^N$, where $N = S + K$.

\begin{definition}[Constraint Manifold]
\label{def:app_manifold}
The constraint manifold is defined as:
\begin{equation}
    \mathcal{M} := \{(\mathbf{x}, \boldsymbol{\mu}) \in \mathcal{D} : \mathbf{c}(\mathbf{x}, \boldsymbol{\mu}) = \mathbf{0}\}.
    \label{eq:app_manifold}
\end{equation}
\end{definition}

\begin{proposition}[Manifold Structure]
\label{prop:app_manifold}
Under Assumption~\ref{asmp:app_smooth} below, $\mathcal{M}$ is an $S$-dimensional submanifold embedded in $\mathbb{R}^N$.
\end{proposition}

\begin{proof}
The Jacobian of $\mathbf{c}$ is $J_c(\mathbf{x}, \boldsymbol{\mu}) = \begin{bmatrix} J_h(\mathbf{x}) & \mathbb{I}_K \end{bmatrix}$, which has constant rank $K$ due to the identity block. By Theorem~\ref{thm:app_level_set}, $\mathcal{M}$ is an embedded submanifold of codimension $K$, hence dimension $N - K = S$.
\end{proof}

\begin{remark}[Projection to safe set]
\label{rmk:app_projection}
The safe set $\mathcal{C}$ is the projection of $\mathcal{M}$ onto the original state space: for any $(\mathbf{x}, \boldsymbol{\mu}) \in \mathcal{M}$, we have $\mathbf{h}(\mathbf{x}) = -\boldsymbol{\mu} \leq \mathbf{0}$, hence $\mathbf{x} \in \mathcal{C}$.
\end{remark}

\paragraph{Tangent space.}
The tangent space of $\mathcal{M}$ at $(\mathbf{x}, \boldsymbol{\mu})$ is:
\begin{equation}
    \mathrm{T}_{(\mathbf{x},\boldsymbol{\mu})}\mathcal{M} = \ker J_c(\mathbf{x}, \boldsymbol{\mu}) = \{\mathbf{v} \in \mathbb{R}^N : J_c(\mathbf{x}, \boldsymbol{\mu})\, \mathbf{v} = \mathbf{0}\}.
\end{equation}
A basis of this tangent space can be expressed in matrix form as $B(\mathbf{x}, \boldsymbol{\mu}) \in \mathbb{R}^{N \times S}$, satisfying $J_c \, B = \mathbf{0}$.

\subsection{Augmented Dynamics and Safe Control Law}
\label{app:control_law}

\paragraph{Slack variable dynamics.}
We endow each slack variable with controlled dynamics:
\begin{equation}
    \dot{\mu}_c = \alpha_c(\mu_c)\, u_{\mu,c}, \quad c \in \{1, \ldots, K\},
    \label{eq:app_slack_dynamics}
\end{equation}
where $\alpha_c: [0, +\infty) \to [0, +\infty)$ is a class-$\mathcal{K}$ function (continuous, strictly increasing, $\alpha_c(0) = 0$) that is locally Lipschitz, and $u_{\mu,c}$ is a virtual control input. In practice, we use the exponential form $\alpha(\mu) = \exp(\beta \mu) - 1$.

\begin{lemma}[Positivity of Slack Variables]
\label{lem:app_slack_positive}
Consider $\dot{\mu} = \alpha(\mu) u_\mu$ with $\alpha$ of class $\mathcal{K}$ and locally Lipschitz, $u_\mu \in [u_\mu^-, u_\mu^+]$ with $u_\mu^- < 0 < u_\mu^+$. For every initial state $\mu(0) > 0$, there exists $\epsilon > 0$ such that $\mu(t) \geq \epsilon$ for all $t \geq 0$.
\end{lemma}

\begin{proof}
Since $\alpha$ is $L$-Lipschitz with $\alpha(0) = 0$, we have $|\alpha(\mu)| \leq L\mu$. For any $u_\mu \in [u_\mu^-, u_\mu^+]$:
$\dot{\mu} = \alpha(\mu) u_\mu \geq \alpha(\mu) u_\mu^- \geq L u_\mu^- \mu$,
where the second inequality uses $\alpha(\mu) > 0$ and $u_\mu^- < 0$. Setting $L' = L u_\mu^-$, the comparison principle gives $\mu(t) \geq \mu(0) e^{L' t} > 0$ for all $t \geq 0$.
\end{proof}

\paragraph{Augmented system.}
Combining the original dynamics~\eqref{eq:ori_sys_dynamics} with the slack dynamics~\eqref{eq:app_slack_dynamics}:
\begin{equation}
    \begin{bmatrix} \dot{\mathbf{x}} \\ \dot{\boldsymbol{\mu}} \end{bmatrix}
    = \begin{bmatrix} f(\mathbf{x}) \\ \mathbf{0} \end{bmatrix}
    + \begin{bmatrix} G(\mathbf{x}) & \mathbf{0} \\ \mathbf{0} & A(\boldsymbol{\mu}) \end{bmatrix}
    \begin{bmatrix} \mathbf{u} \\ \mathbf{u}_\mu \end{bmatrix},
    \label{eq:app_augmented}
\end{equation}
where $A(\boldsymbol{\mu}) = \mathrm{diag}(\alpha_1(\mu_1), \ldots, \alpha_K(\mu_K))$.

\paragraph{Deriving the safe control law.}
To ensure the augmented state remains on $\mathcal{M}$, we require $\dot{\mathbf{c}} = \mathbf{0}$:
\begin{equation}
    \dot{\mathbf{c}} = J_c \begin{bmatrix} \dot{\mathbf{x}} \\ \dot{\boldsymbol{\mu}} \end{bmatrix}
    = \underbrace{J_h(\mathbf{x}) f(\mathbf{x})}_{\boldsymbol{\psi}(\mathbf{x})}
    + \underbrace{\begin{bmatrix} J_h(\mathbf{x}) G(\mathbf{x}) & A(\boldsymbol{\mu}) \end{bmatrix}}_{J_u(\mathbf{x}, \boldsymbol{\mu})}
    \begin{bmatrix} \mathbf{u} \\ \mathbf{u}_\mu \end{bmatrix} = \mathbf{0}.
    \label{eq:app_tangency}
\end{equation}
Here $\boldsymbol{\psi}(\mathbf{x}) = J_h(\mathbf{x}) f(\mathbf{x})$ is the \emph{constraint drift} induced by the system drift. The general solution is:
\begin{equation}
    \begin{bmatrix} \mathbf{u} \\ \mathbf{u}_\mu \end{bmatrix} = -J_u^\dagger \boldsymbol{\psi} + B_u \mathbf{u}_{\mathrm{ref}},
\end{equation}
where $J_u^\dagger$ is the Moore--Penrose pseudoinverse, $B_u$ is a basis of $\ker J_u$ (i.e., $J_u B_u = \mathbf{0}$), and $\mathbf{u}_{\mathrm{ref}}$ is an arbitrary reference input (from the task policy or subgoal tracking controller).

Adding a contraction term to ensure convergence to the manifold from nearby states, the control law is:
\begin{equation}
    \begin{bmatrix} \mathbf{u}_s \\ \mathbf{u}_\mu \end{bmatrix}
    = \underbrace{-J_u^\dagger \boldsymbol{\psi}}_{\text{drift compensation}}
    \underbrace{- \lambda J_u^\dagger \mathbf{c}}_{\text{contraction}}
    + \underbrace{B_u \mathbf{u}_{\mathrm{ref}}}_{\text{tangential}},
    \label{eq:app_atacom}
\end{equation}
where $\lambda > 0$ is the contraction gain. The first term compensates for the constraint drift, the second drives the state toward $\mathcal{M}$ when $\mathbf{c} \neq \mathbf{0}$, and the third enables task-directed motion along the manifold.

\subsection{Single Agent Safety Guarantees}
\label{app:safety_guarantees}
 
This section establishes the main safety result: under mild assumptions, the constraint manifold controller~\eqref{eq:app_atacom} renders the constraint manifold $\mathcal{M}$ attractive, so that every trajectory starting sufficiently close to $\mathcal{M}$ converges to it and, once on $\mathcal{M}$, remains there indefinitely. Since $\mathcal{M}$ projects onto the safe set $\mathcal{C}$ (Remark~\ref{rmk:app_projection}), this guarantees constraint satisfaction $\mathbf{h}(\mathbf{x}(t)) \leq \mathbf{0}$ for all $t \geq 0$.
 
\paragraph{Proof strategy.}
The argument has three ingredients:
\begin{enumerate}
    \item A Lyapunov-like function $V(\mathbf{x}, \boldsymbol{\mu}) = \tfrac{1}{2}\mathbf{c}^\top\mathbf{c}$ that measures the distance from $\mathcal{M}$: $V = 0$ if and only if $(\mathbf{x}, \boldsymbol{\mu}) \in \mathcal{M}$.
    \item A direct computation showing that the constraint manifold controller yields $\dot{V} \leq 0$, so $V$ is non-increasing along trajectories.
    \item An application of LaSalle's Invariance Principle (Theorem~\ref{thm:app_lasalle}) to conclude that trajectories actually converge to $\mathcal{M}$, not just to some lower level set of $V$.
\end{enumerate}
We first state the assumptions needed to make this argument rigorous, then define the region over which convergence is guaranteed, and finally present the theorem and its proof.
 
 
\begin{assumption}[Compact State Space]
\label{asmp:app_compact}
The state space $\mathcal{X}$ is compact and positively invariant with respect to~\eqref{eq:ori_sys_dynamics}.
\end{assumption}
 
\begin{assumption}[Smooth Constraints]
\label{asmp:app_smooth}
The constraint function $\mathbf{h}$ is of class $C^1$.
\end{assumption}
 
\begin{assumption}[Non-empty Safe Set]
\label{asmp:app_nonempty}
The constraint manifold $\mathcal{M}$ defined in~\eqref{eq:app_manifold} is non-empty.
\end{assumption}
 
\begin{assumption}[Constraint Qualification]
\label{asmp:app_rank}
For all $(\mathbf{x}, \boldsymbol{\mu}) \in \partial\mathcal{M}$, the submatrix $(J_h(\mathbf{x}) G(\mathbf{x}))_{[\iota_{h=0},:]}$ has rank $|\iota_{h=0}(\mathbf{x})|$, where $\iota_{h=0}(\mathbf{x}) = \{c : h_c(\mathbf{x}) = 0\}$ is the index set of active constraints.
\end{assumption}
 
Assumptions~\ref{asmp:app_compact}--\ref{asmp:app_nonempty} are standard. Assumption~\ref{asmp:app_rank} is the key constraint qualification: it ensures that $J_u$ has full row rank on $\mathcal{M}$, which in turn guarantees that the linear system~\eqref{eq:app_tangency} is solvable and that $J_u J_u^\dagger = \mathbb{I}_K$.
 
\begin{remark}
Assumption~\ref{asmp:app_rank} implies that the number of simultaneously active constraints cannot exceed the control dimension: $|\iota_{h=0}(\mathbf{x})| \leq U$. For the 2D navigation agents with $U = 2$, at most two constraints can be simultaneously active at the boundary.
\end{remark}
 
\paragraph{Region of contraction.}
 
LaSalle's Invariance Principle requires a compact positively invariant set on which to apply the analysis. We construct such a set using sublevel sets of the Lyapunov function, but must first exclude a pathological set where the contraction mechanism can fail.
 
\begin{definition}[Singular Set]
\label{def:app_singular}
The \emph{singular set} is
$$\mathcal{Y} := \{(\mathbf{x}, \boldsymbol{\mu}) \in \mathcal{D} : J_u^\top \mathbf{c} = \mathbf{0},\, \|\boldsymbol{\mu}\| = 0,\, \|\mathbf{c}\| \neq 0\}.$$
\end{definition}
 
At points in $\mathcal{Y}$, the contraction term $-\lambda J_u^\dagger \mathbf{c}$ vanishes even though $\mathbf{c} \neq \mathbf{0}$, so the controller cannot drive the state back toward $\mathcal{M}$. Crucially, when $\mathcal{Y}$ is non-empty, it is disjoint from $\mathcal{M}$ with positive distance: $\mathrm{dist}(\mathcal{M}, \mathcal{Y}) > 0$.
 
Define the Lyapunov-like function
\begin{equation}
    V(\mathbf{x}, \boldsymbol{\mu}) = \tfrac{1}{2} \mathbf{c}^\top \mathbf{c},
    \label{eq:app_lyapunov}
\end{equation}
and choose $\eta$ such that $0 < \eta < \inf\{V(\mathbf{x}, \boldsymbol{\mu}) : (\mathbf{x}, \boldsymbol{\mu}) \in \mathcal{Y}\}$ (or $\eta$ arbitrarily large if $\mathcal{Y} = \emptyset$). The \emph{region of contraction}
\begin{equation}
    \Omega_\eta := \{(\mathbf{x}, \boldsymbol{\mu}) \in \mathcal{D} : V(\mathbf{x}, \boldsymbol{\mu}) \leq \eta\}
\end{equation}
is compact, contains $\mathcal{M}$ (since $V = 0$ on $\mathcal{M}$), and is disjoint from $\mathcal{Y}$ by construction.
 
\paragraph{Main safety theorem.}
 
The proof of the main theorem relies on the following elementary lemma about pseudoinverses, which allows us to identify the set where $\dot{V}$ vanishes.
 
\begin{lemma}[Pseudoinverse Kernel]
\label{lem:app_pseudoinverse}
Let $X \in \mathbb{R}^{m \times n}$ with $\mathrm{rank}(X) \leq m \leq n$, and let $\mathbf{x} \in \mathbb{R}^m$. If $\mathbf{x}^\top X X^\dagger \mathbf{x} = 0$, then $X^\top \mathbf{x} = \mathbf{0}$.
\end{lemma}
 
\begin{proof}
Let $X = U_1 \Sigma_1 V_1^\top$ be the compact SVD, where $\Sigma_1$ contains the positive singular values. The pseudoinverse is $X^\dagger = V_1 \Sigma_1^{-1} U_1^\top$, so $XX^\dagger = U_1 U_1^\top$ is the orthogonal projector onto the column space of $X$. The assumption $\mathbf{x}^\top U_1 U_1^\top \mathbf{x} = \|U_1^\top \mathbf{x}\|^2 = 0$ gives $U_1^\top \mathbf{x} = \mathbf{0}$. Since $X^\top = V_1 \Sigma_1 U_1^\top$, we conclude $X^\top \mathbf{x} = V_1 \Sigma_1 (U_1^\top \mathbf{x}) = \mathbf{0}$.
\end{proof}
 
\begin{theorem}[Per-Agent Safety~\citep{manifold}]
\label{thm:app_safety}
Consider the augmented system~\eqref{eq:app_augmented} equipped with the constraint manifold controller~\eqref{eq:app_atacom}. Under Assumptions~\ref{asmp:app_compact}--\ref{asmp:app_rank}, and provided a feasible control $\mathbf{u}_s \in \mathcal{U}$ satisfying~\eqref{eq:app_tangency} exists throughout $\Omega_\eta$, the following hold:
\begin{enumerate}
    \item \emph{(Attractivity)} Every trajectory starting in $\Omega_\eta$ approaches the constraint manifold $\mathcal{M}$ as $t \to +\infty$.
    \item \emph{(Forward invariance)} If $(\mathbf{x}(0), \boldsymbol{\mu}(0)) \in \mathcal{M}$, then $(\mathbf{x}(t), \boldsymbol{\mu}(t)) \in \mathcal{M}$ for all $t \geq 0$, and consequently $\mathbf{h}(\mathbf{x}(t)) \leq \mathbf{0}$ for all $t \geq 0$.
\end{enumerate}
\end{theorem}
 
\begin{proof}
We apply LaSalle's Invariance Principle on the compact set $\Omega_\eta$.
 
\textbf{Step 1: Lyapunov decrease.}
Define the Lyapunov-like function $V(\mathbf{s}, \boldsymbol{\mu}) = \tfrac{1}{2}\|\mathbf{c}\|^2$, where $\mathbf{c} = \mathbf{h}(\mathbf{s}) + \boldsymbol{\mu}$. Its time derivative is
\begin{equation}
    \dot{V}
    = \mathbf{c}^\top \dot{\mathbf{c}}
    = \mathbf{c}^\top \Big[
        \boldsymbol{\psi}
        + J_u \begin{bmatrix} \mathbf{u}_s \\ \mathbf{u}_\mu \end{bmatrix}
    \Big].
    \label{eq:app_vdot_raw}
\end{equation}
Substituting the constraint manifold control law~\eqref{eq:app_atacom}:
\begin{align}
    \dot{V}
    &= \mathbf{c}^\top \Big[
        \boldsymbol{\psi}
        + J_u \big(
            -J_u^\dagger \boldsymbol{\psi}
            - \lambda\, J_u^\dagger \mathbf{c}
            + B_u\, \mathbf{u}_{\mathrm{ref}}
        \big)
    \Big] \notag \\
    &= \mathbf{c}^\top \big(
        \underbrace{\boldsymbol{\psi} - J_u J_u^\dagger \boldsymbol{\psi}}_{(a)}
        \underbrace{- \lambda\, J_u J_u^\dagger \mathbf{c}}_{(b)}
        + \underbrace{J_u B_u\, \mathbf{u}_{\mathrm{ref}}}_{(c)}
    \big).
    \label{eq:app_vdot_expand}
\end{align}
Term~(a) vanishes when the tangency condition~\eqref{eq:app_tangency} holds, i.e., $\boldsymbol{\psi} - J_u J_u^\dagger \boldsymbol{\psi} = \mathbf{0}$. Term~(c) vanishes since $J_u B_u = \mathbf{0}$ by construction of the null-space basis. The remaining term gives
\begin{equation}
    \dot{V} = -\lambda\,\mathbf{c}^\top J_u J_u^\dagger\,\mathbf{c} \leq 0,
    \label{eq:app_vdot_final}
\end{equation}
since $J_u J_u^\dagger$ is positive semidefinite and $\lambda > 0$.
 
\textbf{Step 2: Positive invariance of $\Omega_\eta$.}
Since $\dot{V} \leq 0$ throughout $\Omega_\eta$ and $\Omega_\eta = \{(\mathbf{s}, \boldsymbol{\mu}) \in \mathcal{D} : V(\mathbf{s}, \boldsymbol{\mu}) \leq \eta\}$ is compact (Assumption~\ref{asmp:app_compact}), the sublevel set $\Omega_\eta$ is positively invariant.
 
\textbf{Step 3: Attractivity via LaSalle.}
By LaSalle's Invariance Principle, every trajectory starting in $\Omega_\eta$ converges to the largest invariant set contained in
\begin{equation}
    \mathcal{E} := \{(\mathbf{s}, \boldsymbol{\mu}) \in \Omega_\eta : \dot{V} = 0\}
    = \{(\mathbf{s}, \boldsymbol{\mu}) \in \Omega_\eta : \mathbf{c}^\top J_u J_u^\dagger\,\mathbf{c} = 0\}.
    \label{eq:app_lasalle_set}
\end{equation}
By Lemma~\ref{lem:app_pseudoinverse}, $\mathbf{c}^\top J_u J_u^\dagger\,\mathbf{c} = 0$ implies $J_u^\top \mathbf{c} = \mathbf{0}$. Under Assumption~\ref{asmp:app_rank}, $J_u$ has full row rank on $\mathcal{M}$, so $J_u^\top \mathbf{c} = \mathbf{0}$ with $\mathbf{c} \in \mathrm{row}(J_u)$ forces $\mathbf{c} = \mathbf{0}$. Therefore $\mathcal{E} \cap \{\boldsymbol{\mu} \geq \mathbf{0}\} = \mathcal{M}$, and $\mathcal{M}$ is the largest invariant set in $\mathcal{E}$. This establishes attractivity.
 
\textbf{Step 4: Forward invariance of $\mathcal{M}$.}
If $(\mathbf{s}(0), \boldsymbol{\mu}(0)) \in \mathcal{M}$, then $\mathbf{c}(0) = \mathbf{0}$ and $V(0) = 0$. Since $V \geq 0$ and $\dot{V} \leq 0$:
\begin{equation}
    0 \leq V(t) = V(0) + \int_0^t \dot{V}(\tau)\,d\tau
    = \int_0^t \underbrace{\dot{V}(\tau)}_{\leq\,0}\,d\tau \leq 0.
    \label{eq:app_V_zero}
\end{equation}
Hence $V(t) = 0$ for all $t \geq 0$, which implies $\mathbf{c}(t) = \mathbf{0}$. By the projection property (Remark~\ref{rmk:app_projection}), $\mathbf{h}(\mathbf{s}(t)) = -\boldsymbol{\mu}(t) \leq \mathbf{0}$ for all $t \geq 0$.
\end{proof}

\subsection{Multi-agent Safety Guarantees}
\label{app:multi_agent_safety_proof}

\paragraph{Setup.}
Consider $N$ agents with decoupled control-affine dynamics $\dot{\mathbf{s}}_i = f_i(\mathbf{s}_i) + G_i(\mathbf{s}_i)\mathbf{u}_i$. Each agent $i$ maintains a local constraint $\mathbf{h}_i(\mathbf{s}_i, \mathbf{s}_{\mathcal{N}_i}) \leq \mathbf{0}$, where $\mathcal{N}_i = \mathcal{E}_i \cup \mathcal{A}_i$ comprises obstacles $\mathcal{E}_i$ and neighboring agents $\mathcal{A}_i$. The constraint vector stacks pairwise collision terms with all entities in $\mathcal{N}_i$, and each agent applies the constraint manifold safe controller~\eqref{eq:atacom} independently on its local constraint manifold $\mathcal{M}_i$. The global safe set is $\mathcal{C}_{\mathrm{global}} := \bigcap_{i \in \mathcal{N}} \mathcal{C}_i$, where $\mathcal{C}_i = \{\mathbf{S} : \mathbf{h}_i \leq \mathbf{0}\}$.

\paragraph{Assumptions.}
Our analysis requires four conditions, all naturally satisfied by MAS
with distance-based collision constraints:
\begin{enumerate}[label=(\roman*), nosep, leftmargin=*]
    \item The single-agent safety assumptions
    (Assumptions~\ref{asmp:app_compact}--\ref{asmp:app_rank}) hold for
    each agent's local system.
    \item \emph{Decoupled control:} each $\mathbf{u}_i$ affects only
    agent $i$'s dynamics.
    \item \emph{Symmetric inter-agent constraints:} $h_{ij} = h_{ji}$ for every
    agent pair $(i, j)$ with $j \in \mathcal{A}_i$.
    \item \emph{Local observability:} each agent observes the positions
    and velocities of all entities in $\mathcal{N}_i$.
\end{enumerate}

\begin{theorem}[Compositional Safety]
\label{thm:app_compositional_safety}
Under assumptions (i)--(iv), if every agent starts on its local
constraint manifold, i.e.,
$(\mathbf{s}_i(0), \mathbf{s}_{\mathcal{N}_i}(0),
\boldsymbol{\mu}_i(0)) \in \mathcal{M}_i$ for all $i \in \mathcal{N}$,
then every agent remains on its local manifold for all $t \geq 0$, and
the joint trajectory satisfies
$\mathbf{S}(t) \in \mathcal{C}_{\mathrm{global}}$ for all $t \geq 0$.
\end{theorem}

\begin{proof}
The proof extends the single-agent safety result (Theorem~\ref{thm:app_safety})
to the multi-agent setting via a per-agent Lyapunov argument.

\textbf{Step 1: Multi-agent constraint dynamics.}
Each agent $i$ defines a constraint residual
$\mathbf{c}_i = \mathbf{h}_i(\mathbf{s}_i, \mathbf{s}_{\mathcal{N}_i}) + \boldsymbol{\mu}_i$
and a local Lyapunov function $V_i = \tfrac{1}{2}\|\mathbf{c}_i\|^2$.
Differentiating $\mathbf{c}_i$ with respect to time:
\begin{equation}
    \dot{\mathbf{c}}_i
    = \frac{\partial \mathbf{h}_i}{\partial \mathbf{s}_i}\dot{\mathbf{s}}_i
    + \underbrace{
        \sum_{e \in \mathcal{E}_i}
        \frac{\partial \mathbf{h}_i}{\partial \mathbf{s}_e}\dot{\mathbf{s}}_e
      }_{\text{obstacle motion}}
    + \underbrace{
        \sum_{j \in \mathcal{A}_i}
        \frac{\partial \mathbf{h}_i}{\partial \mathbf{s}_j}\dot{\mathbf{s}}_j
      }_{\text{other agents}}
    + \dot{\boldsymbol{\mu}}_i
    = \boldsymbol{\psi}_i + J_{u,i}
    \begin{bmatrix} \mathbf{u}_{s,i} \\ \mathbf{u}_{\mu,i} \end{bmatrix},
    \label{eq:app_cdot_multi}
\end{equation}
where the drift and control Jacobian are
\begin{equation}
    \boldsymbol{\psi}_i
    = \frac{\partial \mathbf{h}_i}{\partial \mathbf{s}_i} f_i(\mathbf{s}_i)
    + \sum_{e \in \mathcal{E}_i}
      \frac{\partial \mathbf{h}_i}{\partial \mathbf{s}_e}\dot{\mathbf{s}}_e
    + \sum_{j \in \mathcal{A}_i}
      \frac{\partial \mathbf{h}_i}{\partial \mathbf{s}_j}\dot{\mathbf{s}}_j,
    \qquad
    J_{u,i}
    = \begin{bmatrix}
        \dfrac{\partial \mathbf{h}_i}{\partial \mathbf{s}_i} G_i(\mathbf{s}_i)
        & \mathbb{I}
    \end{bmatrix}.
    \label{eq:app_psi_Ju_multi}
\end{equation}
Since control is decoupled, $J_{u,i}$ involves only agent~$i$'s own
control input $\mathbf{u}_i$ and slack rate $\dot{\boldsymbol{\mu}}_i$.
The drift $\boldsymbol{\psi}_i$ absorbs all terms not controlled by
agent~$i$: its autonomous dynamics, obstacle motions
($\dot{\mathbf{s}}_e = 0$ for static obstacles, or known trajectories
for dynamic ones), and the velocities of other agents produced by their
own policies. In the single-agent setting $\mathcal{A}_i = \emptyset$
and only the obstacle term remains; the multi-agent extension adds the
inter-agent term, but crucially both enter only the drift.
Agent~$i$ can evaluate $\boldsymbol{\psi}_i$ at each instant via
local observation (Assumption~(iv)).

\textbf{Step 2: Per-agent Lyapunov decrease.}
Substituting the constraint manifold control
law~\eqref{eq:atacom} into $\dot{V}_i = \mathbf{c}_i^\top \dot{\mathbf{c}}_i$:
\begin{equation}
    \dot{V}_i
    = \mathbf{c}_i^\top \bigg(
        \boldsymbol{\psi}_i
        + J_{u,i}\Big(
            -J_{u,i}^\dagger \boldsymbol{\psi}_i
            - \lambda\, J_{u,i}^\dagger \mathbf{c}_i
            + B_{u,i}\, \mathbf{u}_{\mathrm{ref},i}
        \Big)
    \bigg).
    \label{eq:app_vdot_expand_multi}
\end{equation}
Expanding yields three terms:
(a)~the drift residual
$\mathbf{c}_i^\top(\boldsymbol{\psi}_i - J_{u,i}J_{u,i}^\dagger\boldsymbol{\psi}_i)$
vanishes because Assumption~\ref{asmp:app_rank} ensures $J_{u,i}$ has
full row rank, so $J_{u,i}J_{u,i}^\dagger = I$ and the drift is
cancelled exactly---independently of the obstacle and inter-agent
velocity terms in $\boldsymbol{\psi}_i$;
(b)~the tangential term
$\mathbf{c}_i^\top J_{u,i} B_{u,i}\,\mathbf{u}_{\mathrm{ref},i} = 0$
since $B_{u,i}$ spans the null space of $J_{u,i}$;
(c)~the remaining contraction term gives
\begin{equation}
    \dot{V}_i = -\lambda\,\mathbf{c}_i^\top J_{u,i}J_{u,i}^\dagger\,\mathbf{c}_i
    \leq 0,
    \label{eq:app_vdot_multi}
\end{equation}
since $J_{u,i}J_{u,i}^\dagger$ is positive semidefinite and $\lambda > 0$.

\textbf{Step 3: Forward invariance.}
Each agent starts on its local constraint manifold, so
$\mathbf{c}_i(0) = \mathbf{0}$ and $V_i(0) = 0$.
Combined with $V_i \geq 0$ and $\dot{V}_i \leq 0$:
\begin{equation}
    0 \leq V_i(t)
    = V_i(0) + \int_0^t \dot{V}_i(\tau)\,d\tau
    = \int_0^t \underbrace{\dot{V}_i(\tau)}_{\leq\,0}\,d\tau
    \leq 0.
    \label{eq:app_Vi_zero}
\end{equation}
Hence $V_i(t) = 0$ for all $t \geq 0$, which implies
$\mathbf{c}_i(t) = \mathbf{0}$ and, by the projection property
(Remark~\ref{rmk:app_projection}),
$\mathbf{h}_i(t) = -\boldsymbol{\mu}_i(t) \leq \mathbf{0}$
for every agent~$i$.

\textbf{Step 4: Global safety via symmetry.}
For any agent pair $(i,j)$ with $j \in \mathcal{A}_i$, the collision
constraint $h_{ij}$ appears in agent~$i$'s local constraint
set $\mathbf{h}_i$, and $h_{ji} = h_{ij}$ appears in
agent~$j$'s set $\mathbf{h}_j$. Since every agent satisfies
its local constraints (Step~3), every pairwise collision
constraint is enforced, yielding
$\mathbf{S}(t) \in \bigcap_i \mathcal{C}_i
= \mathcal{C}_{\mathrm{global}}$ for all $t \geq 0$.
\end{proof}

\section{Analysis of Residual Safety Violations}
\label{app:violations}

Our compositional safety guarantee (Theorem~\ref{thm:compositional_safety}) 
is established in continuous time, under a fixed constraint set and on-manifold 
initialization. Several implementation factors cause mild deviations from these 
conditions, accounting for the small ($<2\%$) violations observed in practice.

\textbf{(1) Discrete-time integration.} The control law guarantees 
$\dot{V}\le 0$ in continuous time, but each discrete integration step induces 
a drift in the residual $c$ on the order of the step size, momentarily 
violating $h$; the magnitude shrinks as the step size decreases.

\textbf{(2) Neighbor truncation (top-$k$).} Each agent constrains only its 
$k$ nearest entities. Entities beyond $k$ are excluded from $h_i$ and 
receive no guarantee, and the top-$k$ set switches as agents move, so newly 
activated constraints need not satisfy $c=0$ at the switching instant.

\textbf{(3) Degeneracy of Assumption~\ref{asmp:app_rank}.} 
When the number of active constraints exceeds the control dimension $U$, $J_u$ 
loses full row rank and the projection yields a least-squares rather than an 
exact drift cancellation, leaving a small residual drift.

\textbf{(4) Partial observability.} Neighbor states are estimated from local 
observations, and the estimation error enters the drift $\psi_i$, making the 
tangent-space projection slightly inexact.

\section{Additional Experiments}

\subsection{Computational efficiency}
\label{app:Computational_efficiency}

We compare the per-iteration training cost of the QP-based CBF controller and the constraint manifold controller on a single NVIDIA H200 GPU, with all training hyperparameters held constant. As shown in Table~\ref{tab:efficiency}, the constraint manifold controller achieves approximately $14.4\times$ higher throughput than the CBF baseline. This speedup arises because the manifold controller computes safety corrections via a closed-form pseudoinverse projection, whereas the CBF controller requires solving a quadratic program (QP) at every timestep. Over the full 200k training iterations, this translates to an estimated training time of ${\sim}9.4$ hours for the manifold controller versus ${\sim}136.8$ hours for the CBF controller.
 
\begin{table}[h]
\centering
\caption{Per-iteration training speed comparison between the QP-based CBF controller and the constraint manifold controller on a single NVIDIA H200 GPU with identical training hyperparameters.}
\label{tab:efficiency}
\begin{tabular}{lcc}
\toprule
\textbf{Method} & \textbf{Speed} & \textbf{Est.\ Total (200k iters)} \\
\midrule
QP-based CBF         & $\sim$0.4 it/s & ${\sim}136.8$ hours \\
Constraint Manifold  & $\sim$5.9 it/s             & $\sim$9.4 hours \\
\midrule
\textbf{Speedup}     & \multicolumn{2}{c}{$\mathbf{14.4\times}$} \\
\bottomrule
\end{tabular}
\end{table}

\subsection{Additional generalization results}

We evaluate generalization across all environments. As shown in Figure~\ref{fig:generalization_appendix}, \textbf{HMM} achieves the best generalization in safety rate: it maintains near-100\% safety even with a large number of agents or obstacles, whereas other baselines degrade significantly. In terms of success rate, HMM also performs best in most environments.

\label{app:add_generalization_results}

\begin{figure}[t]
    \centering
    \begin{subfigure}[b]{0.24\textwidth}
        \includegraphics[width=\textwidth]{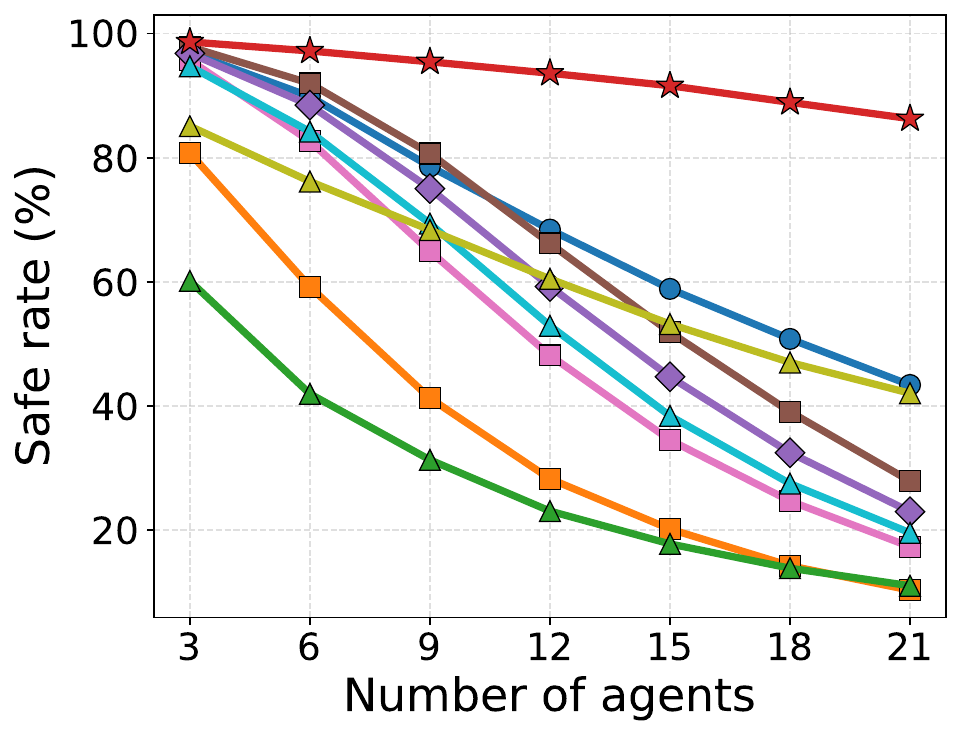}
        \caption{}
    \end{subfigure}
    \hfill
    \begin{subfigure}[b]{0.24\textwidth}
        \includegraphics[width=\textwidth]{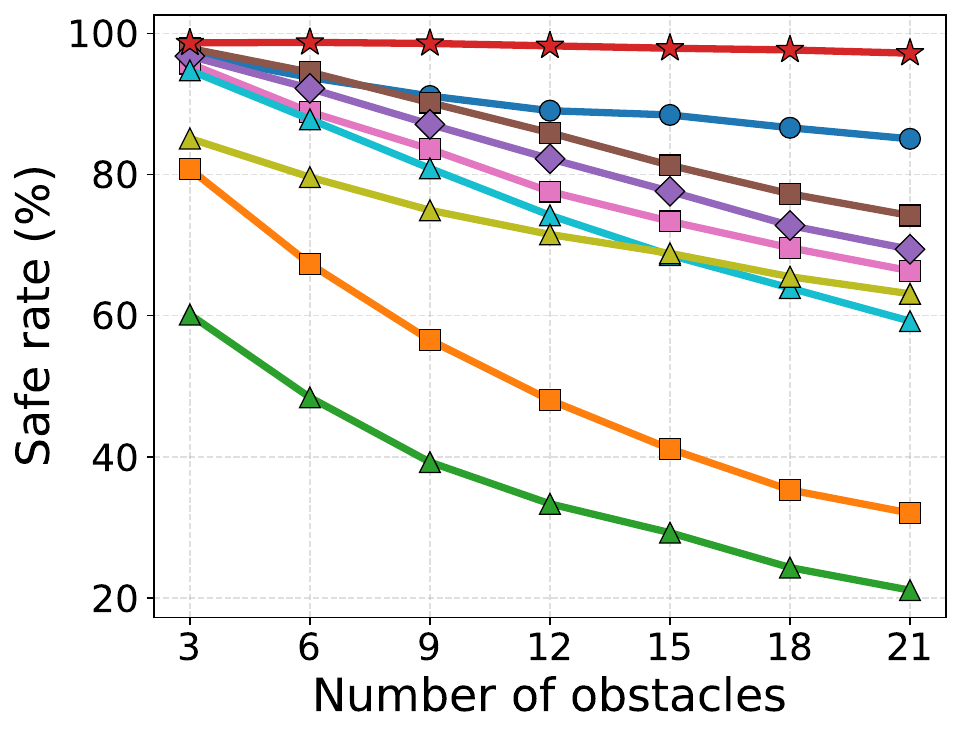}
        \caption{}
    \end{subfigure}
    \hfill
    \begin{subfigure}[b]{0.24\textwidth}
        \includegraphics[width=\textwidth]{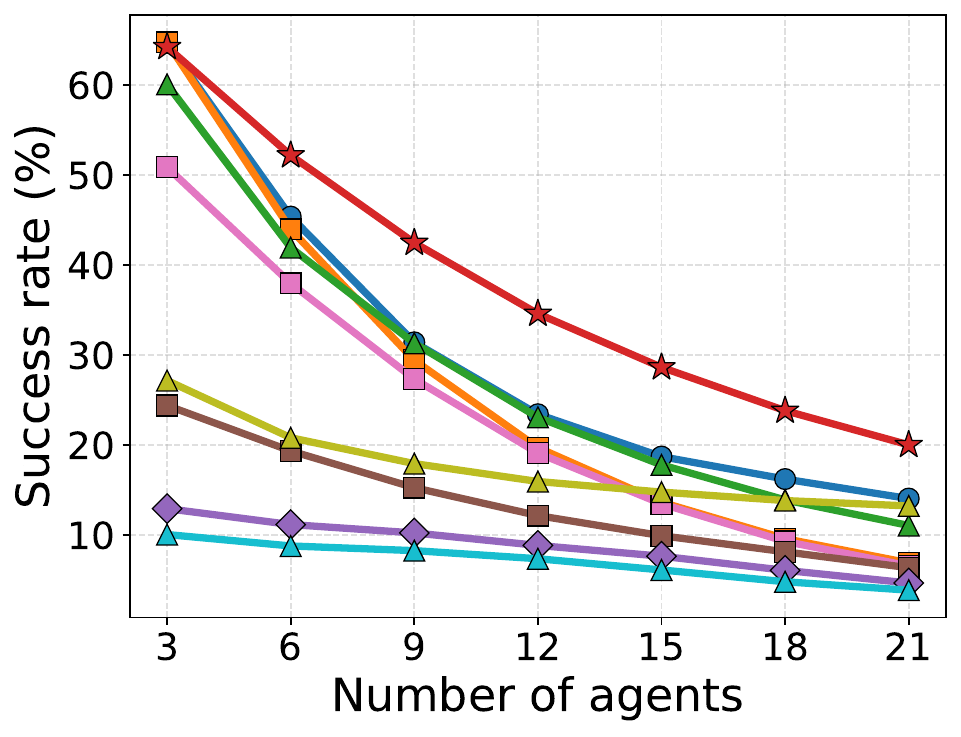}
        \caption{}
    \end{subfigure}
    \hfill
    \begin{subfigure}[b]{0.24\textwidth}
        \includegraphics[width=\textwidth]{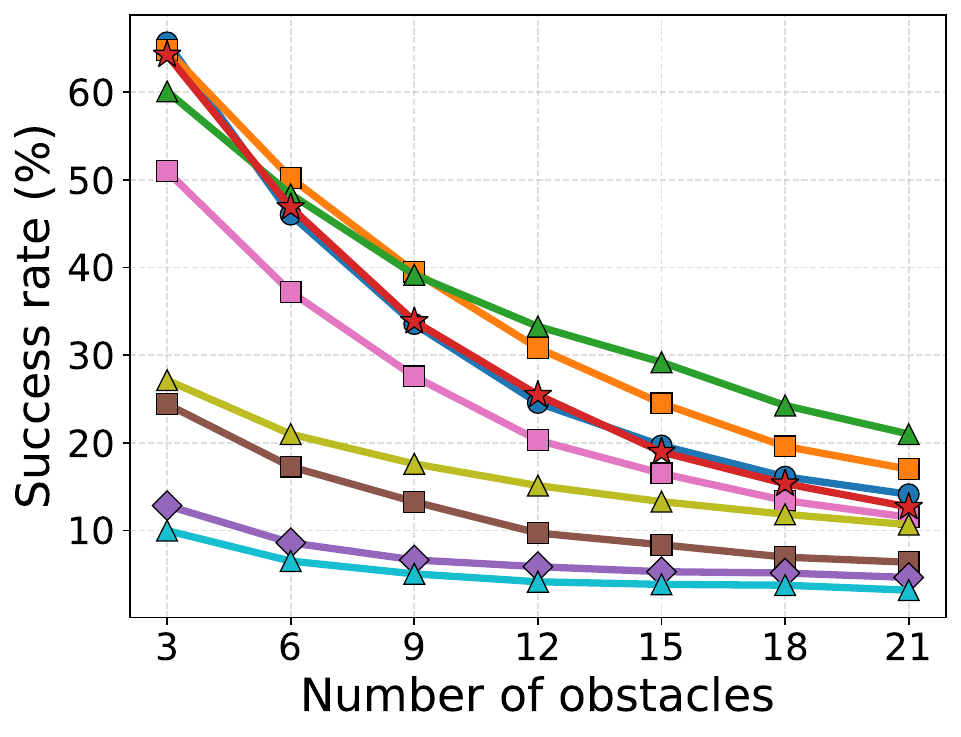}
        \caption{}
    \end{subfigure}

    \vspace{0.3em}

    \begin{subfigure}[b]{0.24\textwidth}
        \includegraphics[width=\textwidth]{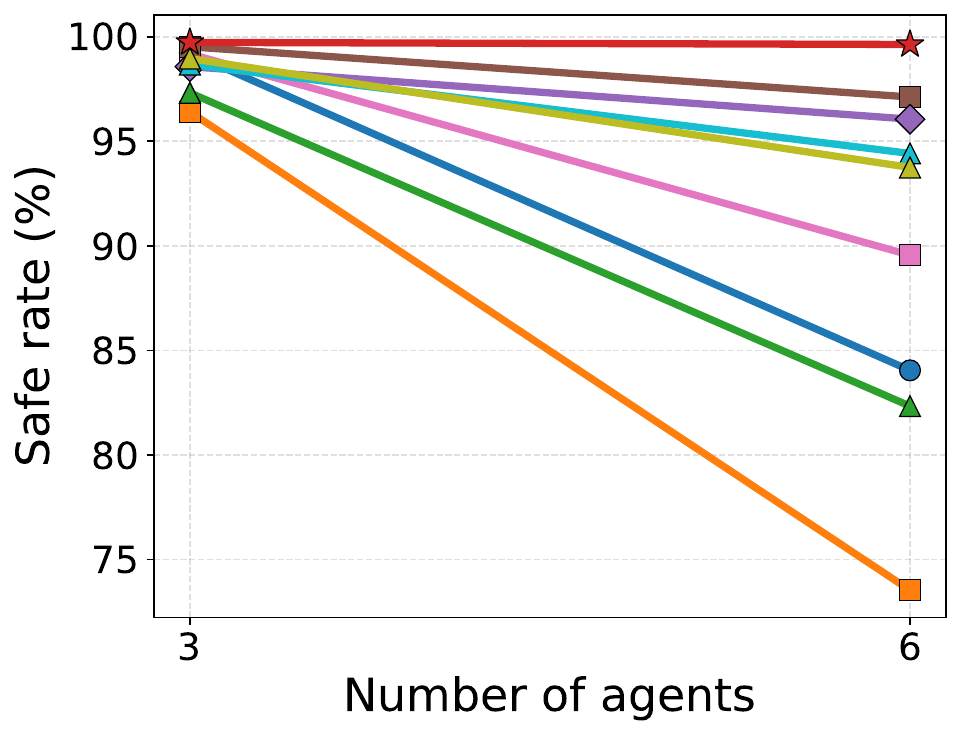}
        \caption{}
    \end{subfigure}
    \hfill
    \begin{subfigure}[b]{0.24\textwidth}
        \includegraphics[width=\textwidth]{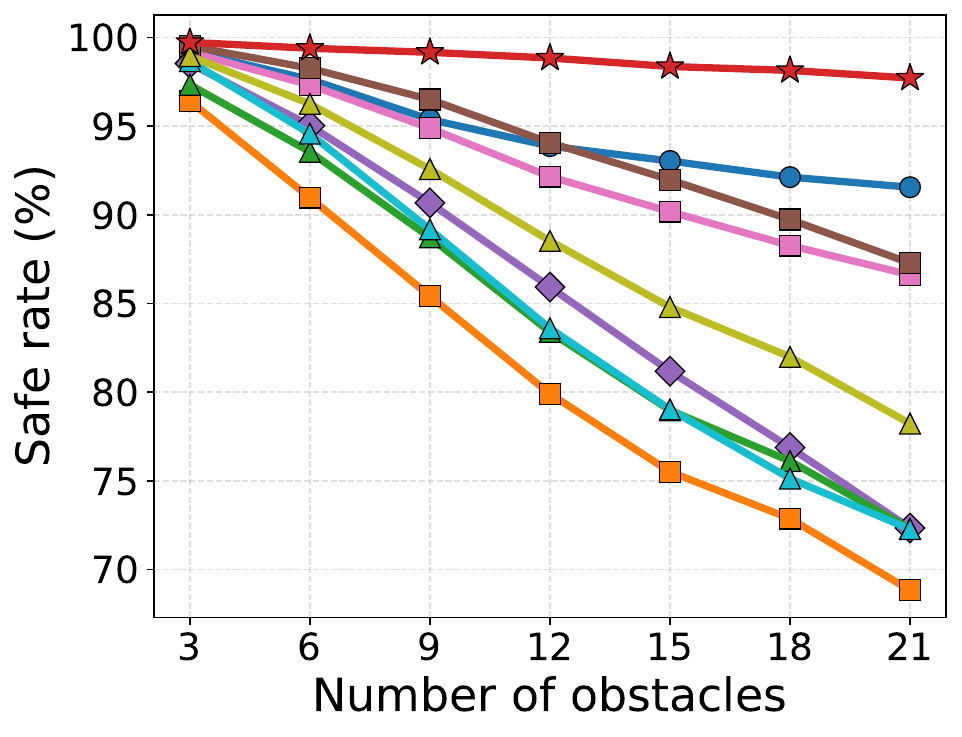}
        \caption{}
    \end{subfigure}
    \hfill
    \begin{subfigure}[b]{0.24\textwidth}
        \includegraphics[width=\textwidth]{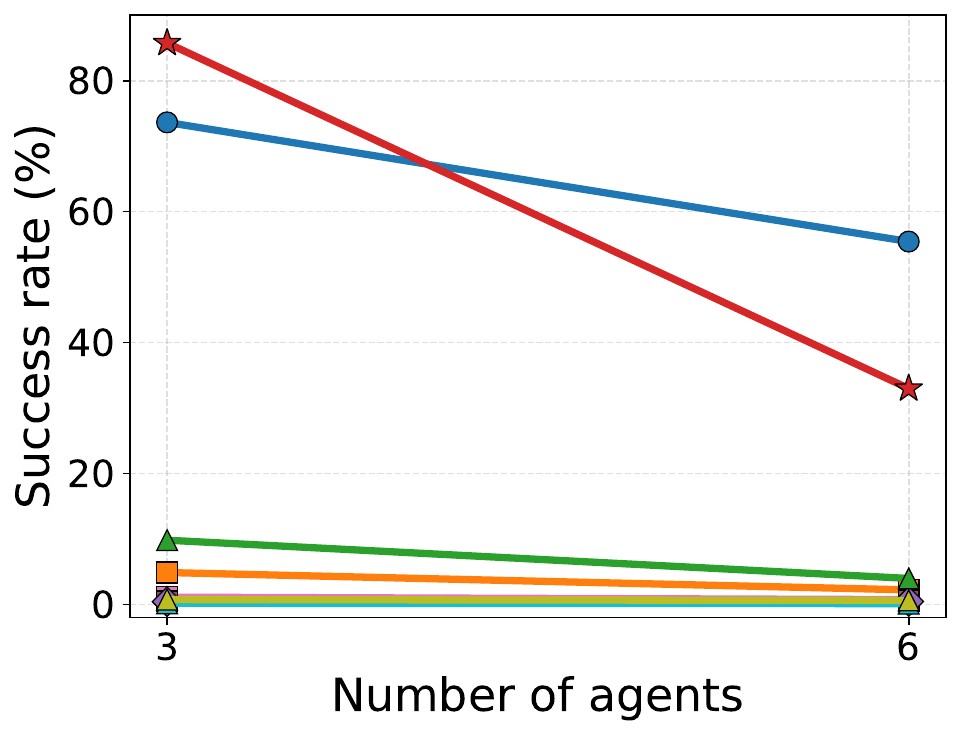}
        \caption{}
    \end{subfigure}
    \hfill
    \begin{subfigure}[b]{0.24\textwidth}
        \includegraphics[width=\textwidth]{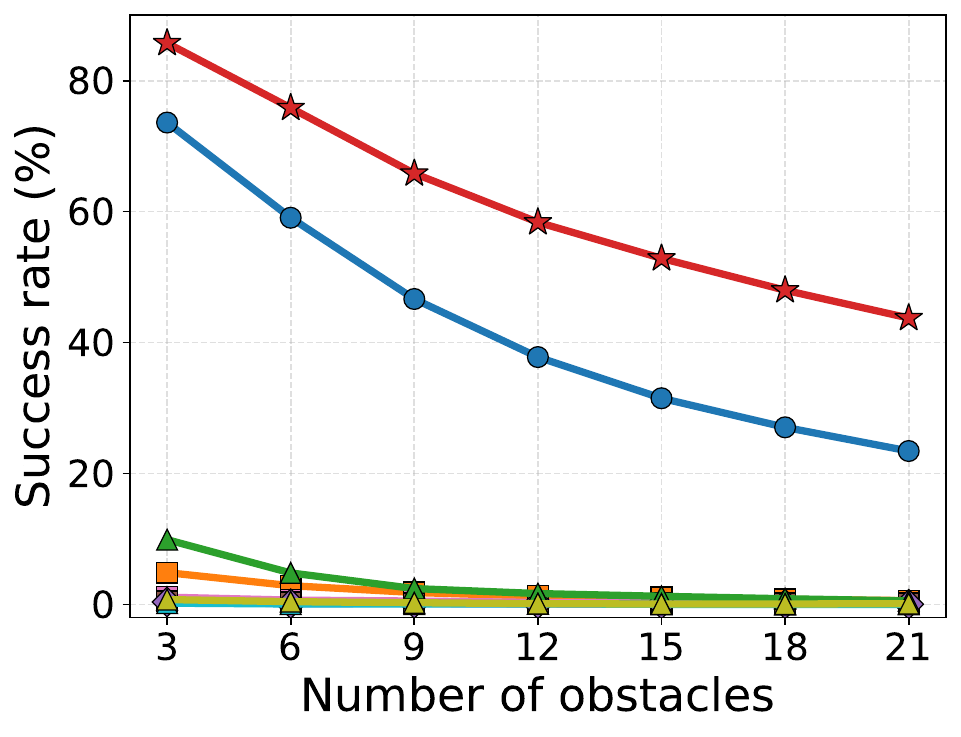}
        \caption{}
    \end{subfigure}

    \vspace{0.3em}

    \begin{subfigure}[b]{0.24\textwidth}
        \includegraphics[width=\textwidth]{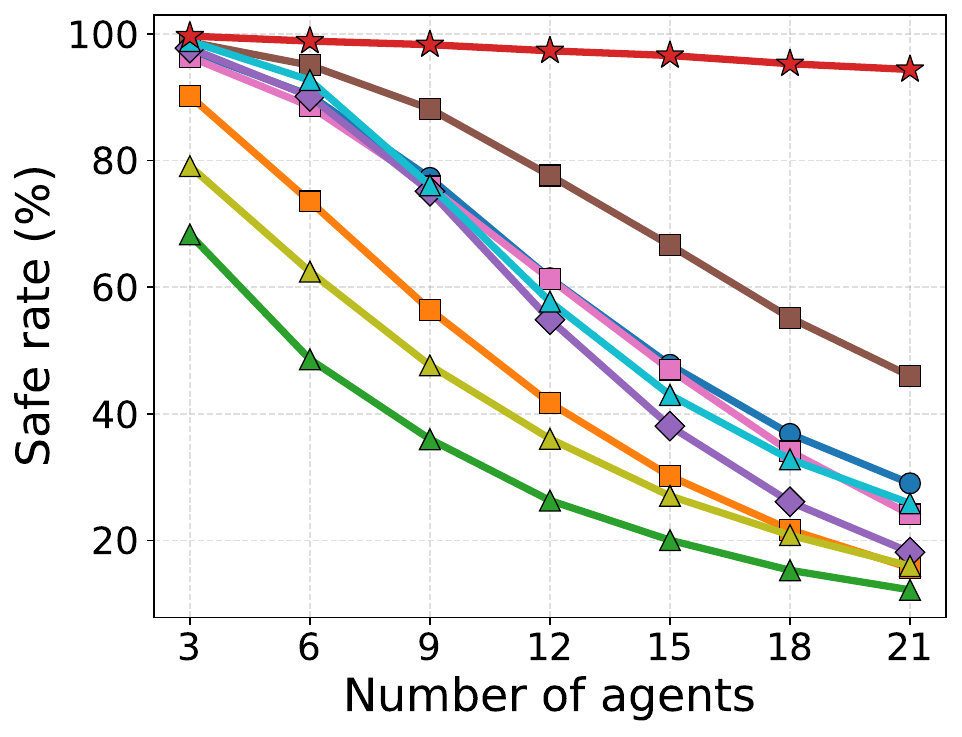}
        \caption{}
    \end{subfigure}
    \hfill
    \begin{subfigure}[b]{0.24\textwidth}
        \includegraphics[width=\textwidth]{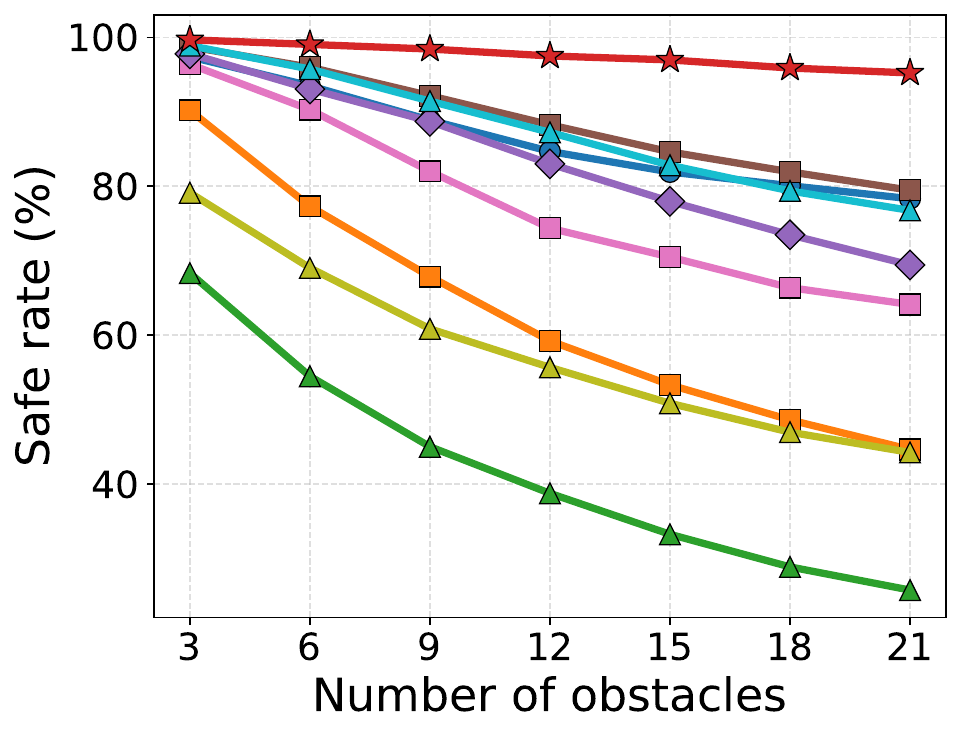}
        \caption{}
    \end{subfigure}
    \hfill
    \begin{subfigure}[b]{0.24\textwidth}
        \includegraphics[width=\textwidth]{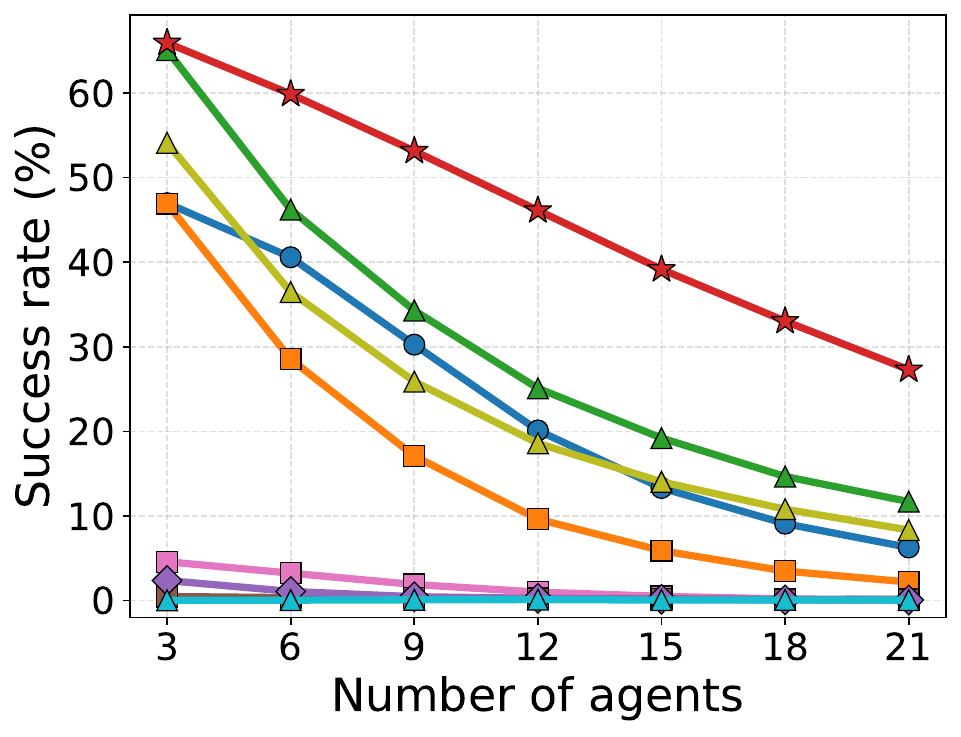}
        \caption{}
    \end{subfigure}
    \hfill
    \begin{subfigure}[b]{0.24\textwidth}
        \includegraphics[width=\textwidth]{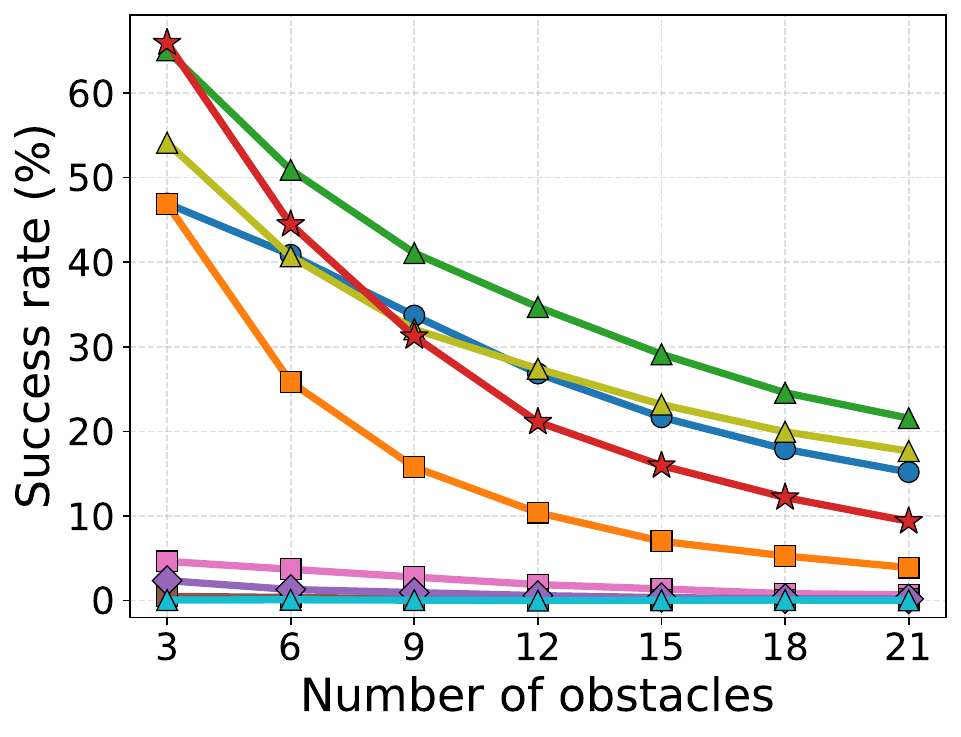}
        \caption{}
    \end{subfigure}

    \vspace{0.3em}

    \begin{subfigure}[b]{0.24\textwidth}
        \includegraphics[width=\textwidth]{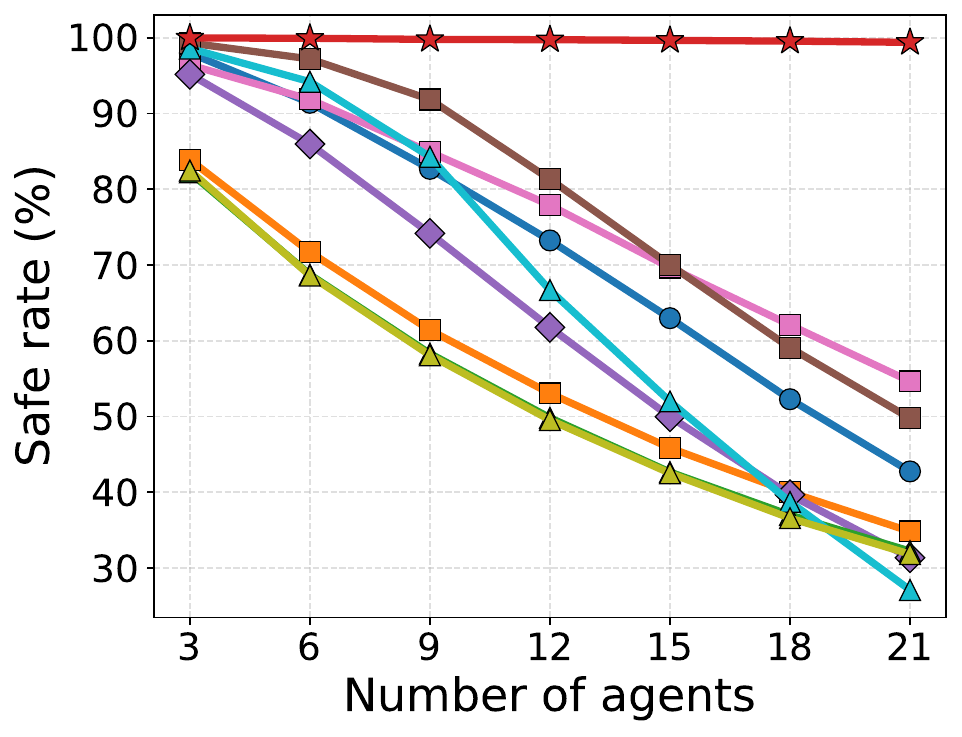}
        \caption{}
    \end{subfigure}
    \hfill
    \begin{subfigure}[b]{0.24\textwidth}
        \includegraphics[width=\textwidth]{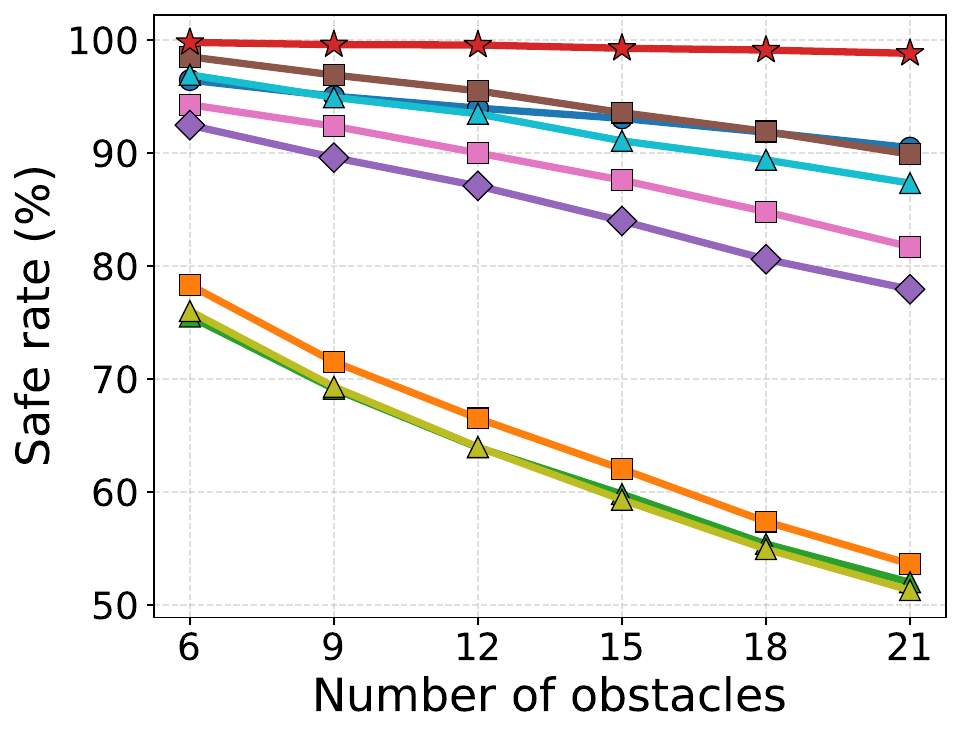}
        \caption{}
    \end{subfigure}
    \hfill
    \begin{subfigure}[b]{0.24\textwidth}
        \includegraphics[width=\textwidth]{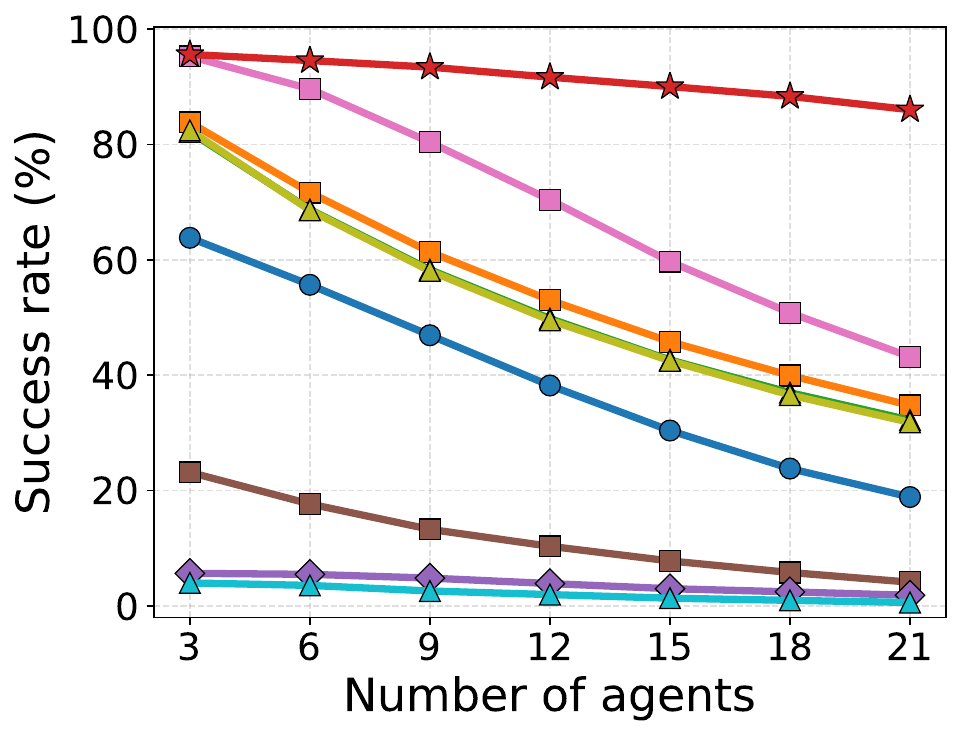}
        \caption{}
    \end{subfigure}
    \hfill
    \begin{subfigure}[b]{0.24\textwidth}
        \includegraphics[width=\textwidth]{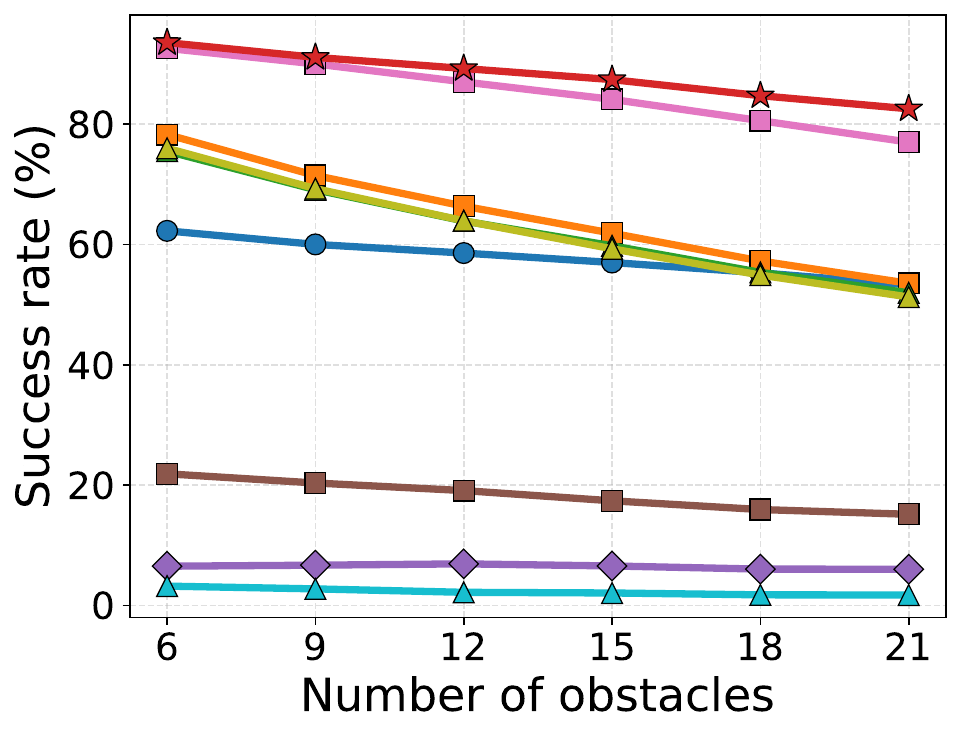}
        \caption{}
    \end{subfigure}

    \includegraphics[width=1\textwidth]{images/generalization/legend.pdf}

    \caption{Generalization results for remaining environments. LidarBicycleTarget (a--d), LidarLine (e--h), LidarTarget (i--l), and LinearDrone (m--p). Left two columns: safe rate; right two columns: success rate.}
    \label{fig:generalization_appendix}
\end{figure}

\section{Hyperparameters}

We summarize the key hyperparameters used in our experiments. Table~\ref{tab:training_hyper} lists the training hyperparameters for all environments, including PPO settings, subgoal configuration, and environment specifications. Table~\ref{tab:manifold_hyper} lists the constraint manifold controller hyperparameters. All values are fixed across environments unless otherwise noted.

\begin{table}[h]
\centering
\caption{Training hyperparameters.}
\label{tab:training_hyper}
\begin{tabular}{lc}
\toprule
\textbf{Parameter} & \textbf{Value} \\
\midrule
Discount $\gamma$ & 0.99 \\
GAE $\lambda$ & 0.95 \\
PPO clip $\epsilon$ & 0.25 \\
Entropy coefficient & $1 \times 10^{-2}$ \\
Actor learning rate & $3 \times 10^{-4}$ \\
Critic learning rate & $1 \times 10^{-3}$ \\
Total environment steps & $2 \times 10^{5}$ \\
Subgoal interval $\tau$ & 8 \\
Relative subgoal max $\Delta$ & 0.2 \\
Episode length $T$ & 128 (3D env: 256) \\
Number of agents & 3 \\
Number of obstacles & 3 (3D env: 6)\\
\bottomrule
\end{tabular}
\end{table}

\begin{table}[h]
\centering
\caption{Constraint manifold controller hyperparameters.}
\label{tab:manifold_hyper}
\begin{tabular}{lcc}
\toprule
\textbf{Parameter} & \textbf{Symbol} & \textbf{Value} \\
\midrule
Top-$k$ nearest neighbors & $k$ & 3 \\
Viability gain & $K$ & 0.5 \\
Error correction gain & $K_c$ & 30.0 \\
Null-space control bound & $\alpha_{\max}$ & 3.0 \\
Constraint activation threshold & $g_{\mathrm{act}}$ & 0.02 \\
Safety margin & -- & 0.02 \\
Slack weight & $w_s$ & 10.0 \\
Slack lower bound & $s_{\min}$ & 0.1 \\
Configuration-space dimension & $d_q$ & 2 (drones: 3) \\
\bottomrule
\end{tabular}
\end{table}



\end{document}